\documentclass[10pt,journal,compsoc]{IEEEtran}

\usepackage{hyperref} 

\usepackage{amsthm}
\usepackage{mathtools}
\usepackage{amssymb}
\usepackage{xcolor}

\ifCLASSOPTIONcompsoc
  \usepackage[nocompress]{cite}
\else
  \usepackage{cite}
\fi
\usepackage{xr}
\makeatletter
\newcommand*{\addFileDependency}[1]{
  \typeout{(#1)}
  \@addtofilelist{#1}
  \IfFileExists{#1}{}{\typeout{No file #1.}}
}
\makeatother

\ifCLASSINFOpdf
\else
\fi
\usepackage{amsmath}
\ifCLASSOPTIONcompsoc
 \usepackage[caption=false,font=footnotesize,labelfont=sf,textfont=sf]{subfig}
\else
 \usepackage[caption=false,font=footnotesize]{subfig}
\fi

\usepackage{stfloats}
\newcommand{\diag}{\operatorname{diag}}
\newcommand{\Id}{\operatorname{\boldsymbol{I}}}

\newcommand{\nbh}{\operatorname{\mathcal{N}}}

\newcommand{\interior}{\operatorname{int}}

\newcommand{\R}{\operatorname{\mathbb{R}}}
\newcommand{\N}{\operatorname{\mathbb{N}}}

\newcommand{\balpha}{\operatorname{\boldsymbol{\alpha}}}
\newcommand{\bPsi}{\operatorname{\boldsymbol{\Psi}}}
\newcommand{\bPhi}{\operatorname{\boldsymbol{\Phi}}}
\newcommand{\bTheta}{\operatorname{\boldsymbol{\Theta}}}
\newcommand{\bLambda}{\operatorname{\boldsymbol{\Lambda}}}

\newcommand{\btheta}{\operatorname{\boldsymbol{\theta}}}

\newcommand{\cL}{\operatorname{\boldsymbol{\mathcal{L}}}}

\newcommand{\bA}{\operatorname{\boldsymbol{A}}}
\newcommand{\bB}{\operatorname{\boldsymbol{B}}}

\newcommand{\bD}{\operatorname{\boldsymbol{D}}}

\newcommand{\bL}{\operatorname{\boldsymbol{L}}}

\newcommand{\bP}{\operatorname{\boldsymbol{P}}}
\newcommand{\bQ}{\operatorname{\boldsymbol{Q}}}

\newcommand{\bU}{\operatorname{\boldsymbol{U}}}

\newcommand{\bW}{\operatorname{\boldsymbol{W}}}
\newcommand{\bX}{\operatorname{\boldsymbol{X}}}
\newcommand{\bY}{\operatorname{\boldsymbol{Y}}}
\newcommand{\bZ}{\operatorname{\boldsymbol{Z}}}

\newcommand{\ba}{\operatorname{\boldsymbol{a}}}

\newcommand{\be}{\operatorname{\boldsymbol{e}}}
\newcommand{\bg}{\operatorname{\boldsymbol{g}}}

\newcommand{\bq}{\operatorname{\boldsymbol{q}}}

\newcommand{\bu}{\operatorname{\boldsymbol{u}}}

\newcommand{\bx}{\operatorname{\boldsymbol{x}}}

\newcommand{\tv}{\tilde{v}}
\newcommand{\tw}{\tilde{w}}
\newcommand{\td}{\tilde{d}}

\newcommand{\tP}{\tilde{P}}
\newcommand{\tU}{\tilde{U}}

\newcommand{\tD}{\tilde{D}}

\newcommand{\tlambda}{\operatorname{\tilde{\lambda}}}

\newcommand{\btX}{\operatorname{\tilde{\bX}}}

\newcommand{\btq}{\operatorname{\tilde{\bq}}}

\newcommand{\bhx}{\boldsymbol{\hat x}}

\newcommand{\bhg}{\boldsymbol{\hat g}}

\newcommand{\low}{\text{low}}
\newcommand{\band}{\text{band}}
\newcommand{\gcn}{\text{gcn}}

\newcommand{\ac}{\text{ac}}
\newcommand{\psct}{\text{-scat}}

\newcommand{\agg}{\operatorname{agg}}
\newcommand{\com}{\operatorname{com}}

\newcommand{\ldcb}{\operatorname{\{\!\!\{}}
\newcommand{\rdcb}{\operatorname{\}\!\!\}}}

\newtheorem{thm}{Theorem}
\newtheorem*{thm*}{Theorem}
\newtheorem{lem}{Lemma}
\newtheorem*{lem*}{Lemma}

\newtheorem{prop}{Proposition}

\theoremstyle{definition}
\newtheorem{defn}{Definition}
\newtheorem*{defn*}{Definition}
\newtheorem{rem}{Remark}

\newtheorem*{ex*}{Example}

\newtheorem{notation}{Notation}

\usepackage{flushend}

\begin{document}
%
\title{Overcoming Oversmoothness in Graph Convolutional Networks via Hybrid Scattering Networks}
%
%
%
%

\author{Frederik~Wenkel$^*$,~
        Yimeng~Min$^*$,~
        Matthew~Hirn,~
        Michael~Perlmutter$^{\dagger}$,~
        and~Guy~Wolf$^{\dagger,**}$~
\thanks{($^*$) The first two authors contributed equally. ($^\dagger$) The last two authors jointly supervised the work. ($^{**}$) Corresponding author.}
\thanks{F. Wenkel and G. Wolf are with Mila (Quebec Artificial Intelligence Institute) and 
Department of Mathematics and Statistics, University of Montreal, Montreal, QC H3T1J4 (e-mail: frederik.wenkel@umontreal.ca, guy.wolf@umontreal.ca)}
\thanks{Y. Min is with the Department
of Computer Science, Cornell University, New York,
NY 14850, USA (e-mail: min@cs.cornell.edu)}
\thanks{M. Hirn is with Department of Computational Mathematics, Science \& Engineering, the Department of Mathematics, and the Center for Quantum Computing, Science \& Engineering, Michigan State University, East Lansing, MI 48824, USA, and the Midwest Quantum Collaboratory (e-mail: mhirn@msu.edu)}
\thanks{M. Perlmutter is with Department of Mathematics, University of California, Los Angeles, CA 90095, USA (e-mail: perlmutter@math.ucla.edu)}
}

%
%

\markboth{~~~~}%
{Shell \MakeLowercase{\textit{et al.}}: Bare Advanced Demo of IEEEtran.cls for IEEE Computer Society Journals}
\IEEEtitleabstractindextext{%
\begin{abstract}
Geometric deep learning has made great strides towards generalizing the design of structure-aware neural networks from traditional domains to non-Euclidean ones, giving rise to graph neural networks (GNN) that can be applied to graph-structured data arising in, e.g., social networks, biochemistry, and material science. Graph convolutional networks (GCNs) in particular, inspired by their Euclidean counterparts, have been successful in processing graph data by extracting structure-aware features. However, current GNN models are often constrained by various phenomena that limit their expressive power and ability to generalize to more complex graph datasets. Most models essentially rely on low-pass filtering of graph signals via local averaging operations, leading to oversmoothing. Moreover, to avoid severe oversmoothing, most popular GCN-style networks tend to be shallow, with narrow receptive fields, leading to underreaching. Here, we propose a hybrid GNN framework that combines traditional GCN filters with band-pass filters defined via geometric scattering. We further introduce an attention framework that allows the model to locally attend over combined information from different filters at the node level. Our theoretical results establish the complementary benefits of the scattering filters to leverage structural information from the graph, while our experiments show the benefits of our method on various learning tasks.
\end{abstract}

}

\maketitle

\IEEEdisplaynontitleabstractindextext

%
\IEEEpeerreviewmaketitle

\section{Introduction}\label{sec:intro}

\IEEEPARstart{D}{eep} learning is typically most effective when the structure of the data can be used to design the architecture of the relevant network. For example, the design of recurrent neural networks is informed by the sequential nature of time-series data. Similarly, the design of convolutional neural networks is based in part on the fact that the pixels of an image are arranged in a rectangular grid. The success of neural networks in these, as well as many other applications, has inspired the rise of geometric deep learning\cite{bronstein2017geometric,bronstein2021geometric}, which aims to extend the success of deep learning to other forms of structured data and to develop intelligent methods for data sets that have a non-Euclidean structure.

A common approach in geometric deep learning is to model the data by a graph. In many applications, this is done by defining edges between data points that interact in a specific way, e.g., ``friends'' on a social network. In many other applications, one may construct a graph from a high-dimensional data set, either by defining an edge between each point and its $k$-nearest neighbors or by defining weighted edges via a similarity kernel.  
Inspired by the increasing ubiquity of graph-structured data, numerous recent works have shown graph neural networks (GNNs) to perform well in a variety of fields including biology, chemistry and social networks \cite{gilmer2017neural,hamilton2017inductive,kipf2016semi}. In these methods, the graph is often considered in conjunction with a set of node features, which contain ``local'' information about, e.g., each user of a social network.

One common family of tasks are  graph-level tasks, where one seeks to learn a whole-graph representation for the purposes of, e.g., predicting  properties of proteins  \cite{fout2017protein,de2018molgan,knyazev2018spectral}. Another common family, which has been the primary focus of graph convolutional networks (GCNs) \cite{kipf2016semi}, are node-level tasks such as node classification. There, the entire data set is modeled as one large graph and the network aims to produce a useful representation of each node using both the node features and the graph structure. This work is typically conducted in the semi-supervised setting where one only knows the labels of a small fraction of the nodes.

Many popular state-of-the-art GNNs essentially aim to promote similarity between adjacent nodes, which may be interpreted as a smoothing operation. While this is effective in certain settings, it can also cause a decrease in performance because of the \emph{oversmoothing} problem~\cite{li2018deeper}, where nodes become increasingly indistinguishable from one another after each subsequent layer. 
In order to partially mitigate the oversmoothing problem, many popular GNNs only use two or three layers. While this does help avoid oversmoothing, it creates a new problem, \emph{underreaching}, where the network is unable to incorporate long-range dependencies or global geometry.

Here, we propose to augment traditional GNN architectures by also including novel band-pass filters, in addition to conventional GCN-style filters\footnote{Throughout this text, we will use the term GCN to refer to the network introduced by Kipf and Welling in \cite{kipf2016semi}. We will use the term GNN to refer to  graph neural networks (spectral, convolutional, or otherwise) in general.} that essentially perform low-pass filtering \cite{nt2019revisiting}, in order to extract richer representations for each node. This approach is based on the geometric scattering transform~\cite{gama2019diffusion,gao2019geometric,zou2020graph}, whose construction is inspired by the Euclidean scattering transform introduced by Mallat in \cite{mallat2012group}, and utilizes iterative cascades of graph wavelets \cite{hammond2011wavelets,coifman2006diffusion} and pointwise nonlinear activation functions to produce useful graph data representations. Notably, these representations are able to capture both high-frequency information and long-range dependencies. 

The main contribution of this work is two hybrid GNN frameworks that utilize both traditional GCN-style filters and also novel filters based on the scattering transform. This approach is based on the following simple idea: GCN-based networks are very useful, but as they aim to enforce similarity among nodes, they essentially focus on low-frequency information. Wavelets, on the other hand, are naturally equipped to capture high-frequency information. In the spatial domain, the GCN-style filters can be thought of as focusing on localized information where the wavelets are able to capture longer-range interactions. Therefore, in a hybrid network, the different channels  capture different types of information. Such a network is therefore more powerful than a network that only uses one style of filter.

We also introduce complementary GNN modules that enhance the performance of such hybrid scattering models, including (i) the graph residual convolution, an adaptive low-pass filter that corrects high-frequency noise, and (ii) an attention framework that enables the aggregation of information from different filters individually at every node. We present theoretical results, based off of a new notion of graph structural difference, that highlight the sensitivity of scattering filters to graph regularity patterns not captured by GCN filters due to underreaching.
Extensive empirical experiments demonstrate the ability of hybrid scattering networks for (transductive) semi-supervised node classification, to (i) alleviate oversmoothing and (ii) generalize to complex (low-homophily) datasets.
Moreover, we also show that our framework translates well to inductive graph-level tasks.

The remainder of this paper is organized as follows. We review related work on GNN models and geometric scattering in Sec.~\ref{sec:related} and introduce important concepts that will be used throughout this work in Sec.~\ref{sec:background}. We then formulate the hybrid scattering network in Sec.~\ref{sec:hybrid}, followed by a theoretical study of the benefits of such models in Sec.~\ref{sec:theory}. In Sec.~\ref{sec:results}, we present empirical results before concluding in Sec.~\ref{sec:conclusion}.


\section{Related Work}\label{sec:related}
Theoretical analyses\cite{li2018deeper,nt2019revisiting} of GCN and related models show that they may be viewed as Laplacian smoothing operations and, from the signal processing perspective, essentially perform low-pass filters on the graph features. One approach towards addressing this problem is the graph attention network proposed by  \cite{velivckovic2018graph}, which uses self-attention mechanisms to address these shortcomings by adaptively reweighting local neighborhoods. In \cite{liao2019lanczosnet},  the authors construct a low-rank approximation of the graph Laplacian that efficiently gathers multi-scale information and demonstrate the effectiveness of their method on citation networks and the QM8 quantum chemistry dataset. In~\cite{abu2019mixhop}, the authors take an approach similar to GCN, but use multiple powers of the adjacency matrix to learn higher-order neighborhood information.
Finally, in \cite{xu2019graph} the authors used graph wavelets to extract higher-order neighborhood.

In addition to the learned networks discussed above, several works\cite{gama2019diffusion,gao2019geometric,zou2020graph} have introduced different variations of the graph scattering transform. These papers aim to extend the Euclidean scattering transform of Mallat \cite{mallat2012group} to graph-structured data and propose predesigned, wavelet-based networks. In \cite{zou2020graph,gama2019diffusion,gama2019stability,perlmutter2019understanding}, extensive theoretical studies of these networks show that they have desirable stability, invariance, and conservation of energy properties. The practical utility of these networks has been established in \cite{gao2019geometric}, which primarily focuses on graph classification, and in \cite{zou2019encoding,castro2020,bhaskar2021molecular}, which used the graph scattering transform to generate molecules.
Building off of these results, which use handcrafted formulations of the scattering transform, recent work~\cite{tong2020data} has proposed a framework for a data-driven tuning of the traditionally handcrafted geometric scattering design that maintains the theoretical properties from traditional designs, while also showing strong empirical results in whole-graph settings.


\section{Geometric Deep Learning Background}\label{sec:background}

\subsection{Graph Signal Processing}
Let $G = (V,E,w)$ be a weighted graph, characterized by a set of nodes (also called vertices) $V\coloneqq \{v_1,\dots,v_n\}$, a set of undirected edges $E\subseteq V\times V$, and a function $w : E \to (0,\infty)$ assigning positive edge weights to the edges. Let $\bX\in\R^{n\times d_X}$ be a \textit{node features matrix}. We shall interpret the $i^\text{th}$ row of $\bX$ as representing the features of the node $v_i$, and therefore, we shall denote these rows by either $\bX[v_i]\in\R^{1\times d_X}$ or $\bX_{v_i}\in\R^{1\times d_X}$. The columns of $\bX$, on the other hand, will be denoted by $\bx_i$. Each of these columns may be naturally identified with a \emph{graph signal}, i.e., a function $x_i: V\rightarrow \mathbb{R}$, $x_i(v_j)\coloneqq\bx_i[j]$. In what follows, for simplicity, we will not distinguish between the vectors $\bx_i$ and the functions $x_i$ and will refer to both as graph signals.

We define the weighted \textit{adjacency matrix} $\bW\in\R^{n\times n}$ of $G$ by
$\bW[v_i,v_j] \coloneqq w(v_i,v_j)$ if $\{v_i,v_j\}\in E$, and set all other entries to zero. We further define the \textit{degree matrix} $\bD\in\R^{n\times n}$ as the diagonal matrix
$\bD\coloneqq \diag(\operatorname{deg}(v_1),\dots, \operatorname{deg}(v_n))$
with each diagonal element $\operatorname{deg}(v_i) \coloneqq \sum_{j=1}^n \bW[v_i,v_j]$ being the \textit{degree} of a node $v_i$. In the following, we will also use the shorthand $d_{v_i}$ to denote the degree of $v_i$. We consider the \textit{combinatorial graph Laplacian} matrix $\bL\coloneqq \bD - \bW$  and the \textit{symmetric normalized Laplacian} given by
$$
    \cL \coloneqq \bD^{-1/2}\bL\bD^{-1/2} = \Id_n - \bD^{-1/2}\bW\bD^{-1/2}.
$$
It is well known that this matrix is symmetric, positive semi-definite, and admits an orthonormal basis of eigenvectors such that $\cL\bq_i=\lambda_i\bq_i.$ Therefore, we may write  
$$
    \cL = \bQ\bLambda\bQ^T = \sum_{i=1}^n \lambda_i \bq_i \bq_i^T,
$$
where $\boldsymbol{\Lambda}\coloneqq \diag(\lambda_1,\dots,\lambda_n)$ and $\bQ$ is the orthogonal matrix whose $i$-th column is $\bq_i$.

We will use the eigendecomposition of $\cL$ to define the \textit{graph Fourier transform}, with the eigenvectors $\bq_1,\dots,\bq_n$ being interpreted as Fourier modes. The notion of oscillation on irregular domains like graphs is delicate, but can be reframed in terms of increasing variation of the modes, with the eigenvalues  $0 = \lambda_1 \leqslant \lambda_2 \leqslant \dots \leqslant \lambda_n \leqslant 2$ interpreted as (squared) frequencies.\footnote{This interpretation originates from motivating the graph Fourier transform via the combinatorial graph Laplacian $\bL=\sum_{i=1}^n \tlambda_i \btq_i \btq_i^T$ with $\tlambda_i=\btq_i^T\bL\btq_i=\sum_{i=1}^n \bW[i,j](\btq_i[i]-\btq_i[j])^2$ the variation of $\btq_i$.}
The \textit{Fourier transform} of a graph signal $\bx\in \R^n$ is defined by $\bhx[i] \coloneqq \langle \bx, \bq_i \rangle$ for $1\leq i\leq n$ and the inverse Fourier transform may be computed by $\bx = \sum_{i=1}^n \boldsymbol{\hat x}[i] \bq_i$.
It will frequently be convenient to write these equations in matrix form as $\bhx = \bQ^T \bx$ and $\bx = \bQ \bhx$.

We recall that in the Euclidean setting, the convolution of two signals in the spatial domain corresponds to the pointwise multiplication of their Fourier transforms. Therefore, we may define the convolution of a signal $\bx$ with a filter $\bg$ by the rule that $\bg\star\bx$ is the unique vector such that $(\widehat{\bg\star \bx})[i] = \bhg[i]\bhx[i]$ for $1\leq i\leq n$. Applying the inverse Fourier transform, one may verify that
\begin{equation}\label{eq:gsp-conv}
    \bg\star \bx
    = \sum_{i=1}^n \bhg[i] \bhx[i] \bq_i
    = \sum_{i=1}^n \bhg[i] \boldsymbol{q}_i \bq^T_i \bx
    = \bQ \boldsymbol{\widehat{G}} \bQ^T \bx,
\end{equation}
where $\boldsymbol{\widehat{G}} \coloneqq \diag(\bhg) = \diag(\bhg[1], \dots, \bhg[n])$. Hence, convolutional graph filters can be parameterized by considering the Fourier coefficients in $\boldsymbol{\widehat{G}}$. 

\subsection{Spectral Graph Neural Network Constructions}

A \textit{graph filter} is a function $F: \mathbb{R}^{n\times d_X}\rightarrow\mathbb{R}^{n\times d_Y}$ that transforms a node feature matrix $\bX\in\mathbb{R}^{n\times d_X}$ into a new feature matrix $\bY\in\mathbb{R}^{n\times d_Y}$. GNNs typically feature several layers each of which produces a new set of features by filtering the output of the previous layer. We will usually let $\bX^0$ denote the initial node feature matrix, which is the input to the network, and let $\bX^\ell$ denote the node feature matrix after the $\ell$-th layer.

In light of Eq.~\ref{eq:gsp-conv}, a natural way to construct learned graph filters would be to directly learn the Fourier coefficients in $\bhg$. Indeed this was the approach used in the pioneering work of Bruna et al. \cite{bruna2014spectral}.  
However, this approach has several drawbacks.
First, it results in $n$ learnable parameters in each convolutional filter. Therefore, a network using such filters would not scale well to large data sets due to the computational cost. At the same time, such filters are not necessarily well-localized in space and are prone to overfitting~\cite{defferrard2016convolutional}.
Moreover, networks of the form introduced in \cite{bruna2014spectral} typically cannot be generalized to different graphs~\cite{bronstein2017geometric}. 
However, recent work%
~\cite{levie2021transferability} has shown that this latter issue can be overcome by formulating the Fourier coefficients $\bhg[i]$ as smooth functions of the Laplacian eigenvalues $\lambda_i$, $1\leq i\leq n$. In particular, this will be the case for the filters used in the networks considered in this work.

A common approach (e.g., used in~\cite{defferrard2016convolutional,kipf2016semi,susnjara2015accelerated,levie2018cayleynets,liao2019lanczosnet}) to formulate such filters is by using polynomials of the Laplacian eigenvalues to set $\bhg[i] \coloneqq \sum_{j=0}^k \theta_j \lambda_i^j$  (or equivalently  $\boldsymbol{\widehat{G}} = \sum_{j=0}^k \theta_j \boldsymbol{\Lambda}^j$) for some $k\ll n$. It can be verified~\cite{defferrard2016convolutional} that this approach yields convolutional filters that are $k$-localized in space and that can be written  as
$ 
    \bg_{\btheta} \star \bx = \sum_{j=0}^{k} \theta_j \cL^j \bx.
$ 
This reduces the number of trainable parameters in each filter from $n$ to $k+1$ and allows one to perform convolution without explicitly computing the spectral decomposition of the Laplacian, which is expensive for large graphs. 

One particularly noteworthy network that uses this method is~\cite{defferrard2016convolutional}, which writes the filters in terms of the Chebyshev polynomials defined by $T_0(x)=1$, $T_1(x)=x$ and $T_j(x)\coloneqq 2x T_{j-1}(x) - T_{j-2}(x)$.
They first renormalize the eigenvalue matrix $\boldsymbol{\tilde{\Lambda}}\coloneqq 2\bLambda/\lambda_{n}-\Id_n$
and then define $\boldsymbol{\widehat{G}} \coloneqq \sum_{j=0}^{k} \theta_j T_j(\boldsymbol{\tilde{\Lambda}})$. Letting $\boldsymbol{\tilde{\mathcal{L}}}\coloneqq 2\boldsymbol{\mathcal{L}}/\lambda_{n}-\Id_n$, yields
\begin{equation}\label{eq:cheb-conv}
    \bg_{\btheta} \star \bx = \sum_{j=0}^{k} \theta_j T_j(\boldsymbol{\tilde{\mathcal{L}}}).
\end{equation}

\subsection{Graph Convolutional Networks}\label{sec:gcn}
One of the most widely used GNNs is the 
Graph Convolutional Network (GCN)~\cite{kipf2016semi}. This network is derived from the Chebyshev construction\cite{defferrard2016convolutional} mentioned above by setting $k=1$ in Eq.~\ref{eq:cheb-conv}
and approximating $\lambda_n\approx 2$, which yields
\begin{align*}
    \bg_{\theta_0,\theta_1} \star \bx
    & \approx \theta_0 \bx + \theta_1 (\cL-\Id_n) \bx \\
    & = \theta_0 \bx - \theta_1 \bD^{-1/2} \bW \bD^{-1/2} \bx.
\end{align*}
To further reduce the number of trainable parameters, the authors then set $\theta\coloneqq \theta_0=-\theta_1$. The resulting convolutional filter has the form
\begin{equation}\label{eq:gcn-conv}
    \boldsymbol{g}_\theta \star \boldsymbol{x} = \theta \left(\Id_n + \boldsymbol{D}^{-1/2} \boldsymbol{W} \boldsymbol{D}^{-1/2}\right) \boldsymbol{x}.
\end{equation}
One may verify that $\Id_n + \boldsymbol{D}^{-1/2} \boldsymbol{W} \boldsymbol{D}^{-1/2}=2\Id_n-\boldsymbol{\mathcal{L}}$, and therefore, Eq.~\ref{eq:gcn-conv} essentially corresponds to setting $\boldsymbol{\hat{g}}[i] \coloneqq \theta (2 - \lambda_i)$ in Eq.~\ref{eq:gsp-conv}. The eigenvalues of $\Id_n + \boldsymbol{D}^{-1/2} \boldsymbol{W} \boldsymbol{D}^{-1/2}$ take values in $[0,2].$ Thus, to avoid  vanishing or exploding gradients, the authors use a renormalization trick
\begin{equation}\label{eq:renorm}
    \Id_n + \boldsymbol{D}^{-1/2} \boldsymbol{W} \boldsymbol{D}^{-1/2} \rightarrow \boldsymbol{\tilde D}^{-1/2} \boldsymbol{\tilde W} \boldsymbol{\tilde D}^{-1/2},
\end{equation}
where $\boldsymbol{\tilde W}\coloneqq \Id_n + \boldsymbol{W}$ and $\boldsymbol{\tilde D}$ is a diagonal matrix with $\boldsymbol{\tilde D}[v_i,v_i] \coloneqq \sum_{j=1}^n \boldsymbol{\tilde W}[v_i,v_j]$ for $i\in[n]$.
Setting
$
    \boldsymbol{A} \coloneqq \boldsymbol{\tilde D}^{-1/2} \boldsymbol{\tilde W} \boldsymbol{\tilde D}^{-1/2},
$
and using multiple channels we obtain a layer-wise propagation rule of the form
$
 \bx_j^\ell = \sigma \big( \sum_{i=1}^{N_{\ell-1}} \theta_{ij}^\ell \bA \bx_i^{\ell-1} \big),
$
where $N_\ell$ is the number of channels used in the $\ell$-th layer and $\sigma(\cdot)$ is an elementwise nonlinearity. In matrix form we write
\begin{equation}\label{eq:gcn}
    \bX^\ell = F_{\gcn}\left(\bX^{\ell-1}\right) = \sigma \left( \bA \bX^{\ell-1} \bTheta^\ell \right).
\end{equation}
We interpret the matrix $\bA$ as computing a localized average of each channel $\bx_i^{\ell-1}$  around each mode and the matrix  $\bTheta$ as sharing information across channels. This filter can also be observed at the node level as
\begin{equation*}
    \bX^{\ell}[v] = \sigma\left(\sum_{w\in\nbh_{\underline{v}}} \frac{1}{\sqrt{(d_v+1)(d_w+1)}} \bX^{\ell-1}[w] \bTheta^\ell \right),
\end{equation*}
where $\nbh_v$ denotes the one-step neighborhood of node $v$ and $\nbh_{\underline{v}} \coloneqq \nbh_v \cup \{v\}$.
This process can be split into three steps:
\begin{subequations}\label{eq:gcn-node-}
    \begin{align}
        & \bX_a^\ell[w] = \bX^{\ell-1}[w] \bTheta^\ell \text{ for all } w\in\nbh_{\underline{v}} \label{eq:gcn-node-a} \\
        & \bX_b^\ell[v]  = \sum_{w\in\nbh_{\underline{v}}} \frac{1}{\sqrt{(d_v+1)(d_w+1)}} \bX_a^{\ell}[w] \label{eq:gcn-node-b} \\
        & \bX_{c}^\ell[v] = \sigma\left(\bX_b^\ell[v]\right) \label{eq:gcn-node-c},
    \end{align}
\end{subequations}
which we refer to as the \textit{transformation} step (Eq.~\ref{eq:gcn-node-a}), the \textit{aggregation} step (Eq.~\ref{eq:gcn-node-b}) and the \textit{activation} step (Eq.~\ref{eq:gcn-node-c}).

As discussed earlier, the GCN filter described above may be viewed as a low-pass filter that suppresses high-frequency information. For simplicity, we focus on the convolution in Eq.~\ref{eq:gcn-conv} before the renormalization.  This convolution essentially corresponds to point-wise Fourier multiplication by~$\bhg[i]=\theta(2-\lambda_i)$, which is strictly decreasing in~$\lambda_i$. Therefore, repeated applications of this filter effectively zero-out the higher frequencies. This is consistent with the oversmoothing problem discussed in~\cite{li2018deeper}.

\subsection{Graph Attention Networks}

Another popular network that is widely used for node classification tasks is the graph attention network (GAT)~\cite{velivckovic2018graph}, which uses an attention mechanism to guide and adjust the aggregation of features from adjacent nodes. First, the node features $\bX^{\ell-1}$ are linearly transformed to $\boldsymbol{\bar X}^\ell\coloneqq \bX^{\ell-1} \bTheta$ using a learned weight matrix $\bTheta \in \mathbb{R}^{d \times d'}$. Then, the aggregation coefficients are learned via
\begin{equation*}
    \alpha_{v\leftarrow u} = \frac{\exp(\operatorname{LeakyReLU}({\left[\boldsymbol{\bar X}^\ell[v]\mathbin\Vert\boldsymbol{\bar X}^\ell[u] \right] \ba})}{\sum_{w \in \nbh_{\underline{v}}} \exp(\operatorname{LeakyReLU}(\left[\boldsymbol{\bar X}^\ell[v]\mathbin\Vert \boldsymbol{\bar X}^\ell[w] \right] \ba)},
\end{equation*}
where $\ba \in \R^{2d'}$ is a shared attention vector and $\Vert$ denotes horizontal concatenation. The output feature corresponding to a single attention head is given by $\bX^\ell[v] = \sigma(\sum_{u \in \nbh_{v}}  \alpha_{v\leftarrow u}\boldsymbol{\bar X}^\ell[u])$. 
To increase the expressivity of this network, the authors then use a  multi-headed attention mechanism, with $\Gamma$ heads, to generate concatenated features
\begin{equation}\label{eq:gat}\textstyle
    \bX^\ell[v] = \mathbin\Vert_{\gamma=1}^\Gamma \sigma\left({\sum_{u\in \nbh_{\underline{v}}}} \alpha_{v\leftarrow u}^\gamma \boldsymbol{\bar X}_\gamma^\ell[u]\right).
\end{equation}

\subsection{Challenges in Geometric Deep Learning}\label{sec:challenges}

Many GNN models, including GCN~\cite{kipf2016semi} and GAT~\cite{velivckovic2018graph}, are subject to the so-called \textit{oversmoothing} problem~\cite{li2018deeper}, caused by aggregation steps (such as Eq.~\ref{eq:gcn-node-b}) that essentially consist of localized averaging operations. As discussed in Sec.~\ref{sec:gcn} and also \cite{nt2019revisiting}, from a signal processing point of view, this corresponds to a \textit{low-pass filtering} of graph signals. Moreover, as discussed in \cite{barcelo2020logical}, these networks are also subject to \textit{underreaching}. Most GNNs (including GCN and GAT) can only relate information from nodes within a distance equal to the number of GNN layers, and because of the oversmoothing problem, they typically use a small number of layers in practice. Therefore, the oversmoothing and underreaching problems combine to significantly limit the ability of GNNs to capture long-range dependencies. In Sec.~\ref{sec:hybrid}, we will introduce a hybrid network, which aims to address these challenges by using both GCN-style channels and channels based on the geometric scattering transform discussed below. 

\subsection{Geometric Scattering}

In this section, we review the geometric scattering transform constructed in \cite{gao2019geometric} for graph classification and show how it may be adapted for node-level tasks. As we shall see, this node-level geometric scattering will address the challenges discussed above in Sec.~\ref{sec:challenges}, by using band-pass filters that capture high-frequency information and have wider receptive fields.

The geometric scattering transform uses wavelets based upon raising the  \textit{lazy random walk} matrix
$$
    \boldsymbol{P} \coloneqq \frac{1}{2} \big( \Id_n + \boldsymbol{W} \boldsymbol{D}^{-1} \big),
$$
to dyadic powers $2^j$, which can be interpreted as differing degrees of resolution.
The right eigenvectors and eigenvalues of $\bP$ are $\bu_i \coloneqq \bD^{1/2}\bq_i$ and $\omega_i \coloneqq 1 - \lambda_i/2, 1\leq i\leq n$, respectively.
Entrywise, we note that 
\begin{equation}\label{eq:P-node}
    (\bP\bX)[v] = \frac{1}{2} \bX[v] + \frac{1}{2} \sum_{w\in\nbh_v} d_w^{-1} \bX[w].
\end{equation}
Thus, $\bP$ may be viewed as a localized averaging operation operator analogous to those used in, e.g., GCN, and the powers $\bP^{2^j}$ may be viewed as low-pass filters which suppress high-frequencies.
In order to better retain this high-frequency information, we define multiscale \textit{graph diffusion wavelets} by subtracting these low-pass filters at different scales~\cite{coifman2006diffusion}. Specifically, for $k\in\mathbb{N}_0$, we define a wavelet  $\boldsymbol{\Psi}_k\in\mathbb{R}^{n\times n}$ at scale $2^k$ by
\begin{equation}\label{eq:wavelet}
    \bPsi_0 \coloneqq \Id_n - \bP, \quad
    \bPsi_k \coloneqq \bP^{2^{k-1}} - \bP^{2^k}, \quad k\geq 1.
\end{equation}
We may interpret each $\boldsymbol{\Psi}_k$ as capturing information at a different frequency band. From a spatial perspective, we may view $\bPsi_k$ as encoding information on how a $2^k$-step neighborhood differs from a smaller $2^{k-1}$-step one. Such wavelets are usually organized in a filter bank \begin{equation}\label{eqn: W}\{\boldsymbol{\Psi}_k, \boldsymbol{\Phi}_K\}_{0\leq k\leq K},\end{equation} along with a low-pass filter $\boldsymbol{\Phi}_K\coloneqq \boldsymbol{P}^{2^K}$.
 Proposition 2.2 of ~\cite{perlmutter2019understanding} (restated in here as Proposition \ref{prop:frame}) shows that this  filter bank is a self-adjoint, non-expansive frame on a suitably weighted inner-product space.

\begin{figure}[!t]
\centering
\includegraphics[width=\columnwidth]{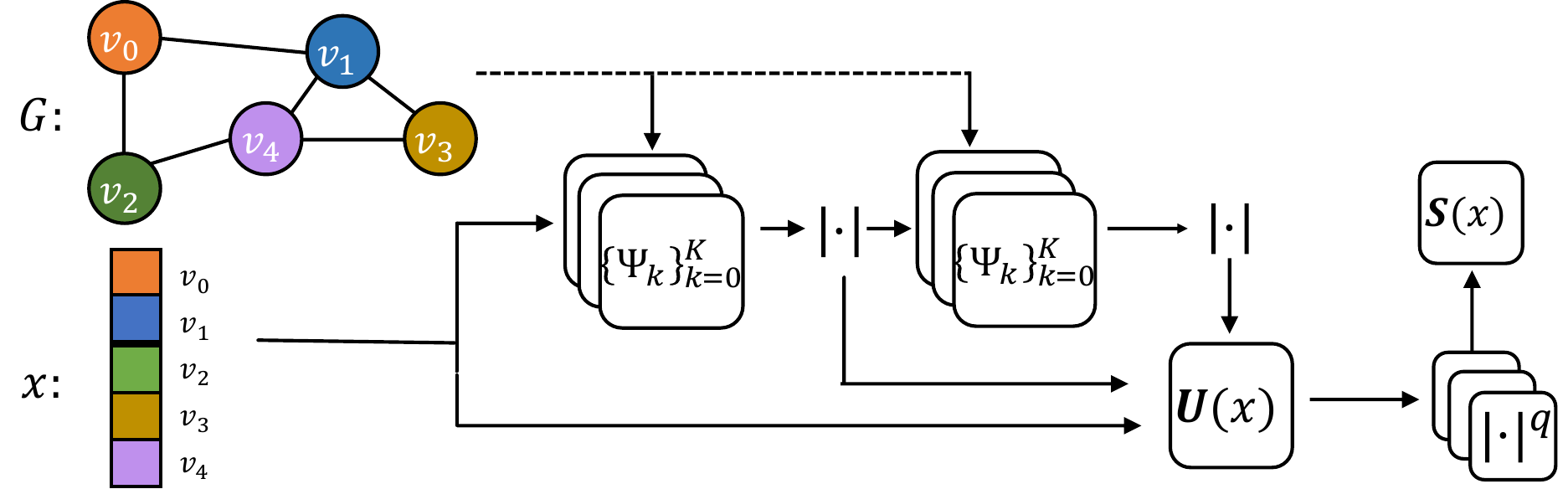}
\caption{Illustration of two-layer geom. scattering at the node level ($\boldsymbol{U}(\boldsymbol{x}) = \{\boldsymbol{U}_p \boldsymbol{x} : p \in \mathbb{N}_0^{m}, m = 0,1,2\}$) and at the graph level ($\boldsymbol{S}(\boldsymbol{x}) = \{\boldsymbol{S}_{p,q} \boldsymbol{x} : q \in \mathbb{N}, p \in \mathbb{N}_0^{m}, m = 0,1,2\}$), extracted according to the wavelet cascade in Eq.~\ref{eq:wavelet}-\ref{eq:scat-graph}. }
\label{fig:GeoSct_illustration}
\end{figure}

The geometric scattering transform is a multi-layered architecture that iteratively applies wavelet convolutions and nonlinear activation functions in an alternating cascade as illustrated in Fig.~\ref{fig:GeoSct_illustration}. It is parameterized by paths $p \coloneqq (k_1, \dots, k_m)$. Formally, we define
\begin{equation}
    \label{eq_original node.scattering}
    \boldsymbol{U}_p \boldsymbol{x} \coloneqq \boldsymbol{\Psi}_{k_m} \circ \sigma \circ \boldsymbol{\Psi}_{k_{m-1}} \cdots \sigma \circ \boldsymbol{\Psi}_{k_2} \circ \sigma \circ \boldsymbol{\Psi}_{k_1}\boldsymbol{x},
\end{equation}
where $\sigma$ is a nonlinear activation function.\footnote{In a slight deviation from previous work, here $\boldsymbol{U}_p$ does not include the outermost nonlinearity in the cascade.} We note that in other work focusing on graph-classification, the scattering features are frequently aggregated into $q^{\text{th}}$-order moments,
\begin{equation}\label{eq:scat-graph} \textstyle
    \boldsymbol{S}_{p,q} \boldsymbol{x} \coloneqq \sum_{i=1}^n \vert (\bU_p \bx) [v_i] \vert^q.
\end{equation} We also note that the nonlinearity $\sigma$ might vary in each step of the cascade. However, we will suppress this possible dependence to avoid cumbersome notation. We also note that in our theoretical results, if we assume, for example, that $\sigma$ is strictly increasing, this assumption is intended to apply to all nonlinearities used in the cascade. In our experiments, we use the absolute value, i.e., $\sigma(\bx[v_i])=\lvert\bx[v_i]\rvert$.

The original formulations of geometric scattering were fully designed networks without any learned convolutions between channels. Here, we will incorporate learning by defining the following scattering propagation rule similar to the one used by GCN:
\begin{equation}\label{eq:scat-node}
    \bX^\ell = F_{p\psct}\left(\bX^{{\ell-1}}\right) = \sigma \left( \bU_{p} \big(\bX^{\ell-1} \bTheta^\ell\big)\right).
\end{equation}
Analogously to GCN, we note that for $v\in V$, the scattering propagation rule can be decomposed into three steps:
\begin{subequations}\label{eq:scat-node-}
    \begin{align}
        & \bX_{a'}^\ell[v] = \bX^{\ell-1}[v] \bTheta^\ell \label{eq:scat-node-a} \\
        & \bX_{b'}^\ell[v]=\left(\bU_p(\bX_{a'}^{\ell})\right)[v] \label{eq:scat-node-b} \\
        & \bX_{c'}^\ell[v]=\sigma(\bX_{b'}^\ell[v])^q \label{eq:scat-node-c}.
    \end{align}
\end{subequations}
Importantly, we note that the scattering transform addresses the underreaching problem as wavelets $\boldsymbol{\Psi}_k=\bP^{2^{k-1}}-\bP^{2^{k}}$ that are leveraged in the aggregation step (Eq.~\ref{eq:scat-node-b}) have larger receptive fields than most traditional GNNs. However, scattering does not result in oversmoothing because the subtraction results in band-pass filters rather than low-pass filters. In this manner, the scattering transform addresses  the challenges discussed in the previous subsection.

\section{Hybrid Scattering Networks}\label{sec:hybrid}

Here, we introduce two hybrid networks that combine aspects of GCN and the geometric scattering transform discussed in the previous section. Our networks will use  both low-pass and band-pass filters in different channels to capture different types of information. As a result, our hybrid networks will have greater expressive power than either traditional GCNs, which only use low-pass filters, or a pure scattering network, which only uses band-pass filters.

In our low-pass channels, we use modified GCN filters, which are similar to those used in Eq.~\ref{eq:gcn} but use higher powers of $\bA$ and include bias terms. Specifically, we use a channel update rule of the form
\begin{equation}\label{eq:gcn-mod}
    \bX_{\text{low},i}^\ell \coloneqq \sigma \left( \bA^{r_i} \bX^{\ell-1} \bTheta_{\text{low},i}^\ell + \bB_{\text{low},i}^\ell \right)
\end{equation}for $1\leq i \leq C_{\text{low}}$. We note, in particular, that the use of higher powers of $\bA$ enables a wider receptive field (of radius $r_i$), without increasing the number of trainable parameters (unlike in GCN).

Similarly, in our band-pass channels, we use a version of Eq.~\ref{eq:scat-node},  with an added bias term, and our update rule is 
\begin{equation}\label{eq:scat-mod}
    \bX_{\text{band},i}^\ell \coloneqq \sigma \left( \bU_{p_i} \big(\bX^{\ell-1} \bTheta_{\text{band},i}^\ell\big) + \bB_{\text{band},i}^\ell \right)
\end{equation}
 for $1\leq i\leq C_{\text{band}}$. Here, similarly to Eq.~\ref{eq_original node.scattering}, $p_i$ is a path that determines the cascade of wavelets used in the $i$-th channel. 
%
%
\\

\noindent\textbf{Aggregation module.}
Each \textit{hybrid layer} uses $C_{\text{low}}$ low-pass and $C_{\text{band}}$ band-pass channels, described in Eq.~\ref{eq:gcn-mod} and Eq.~\ref{eq:scat-mod}, to transform the node features $\bX^{\ell-1}$. The resulting 
$
    \{\bX_{\text{low},i}^\ell\}_{i=1}^{C_{\text{low}}} \text{ and } \{\bX_{\text{band},i}^\ell\}_{i=1}^{C_{\text{band}}}
$
are aggregated to new $d_\ell$-dimensional representations $\bX^\ell$ via an \textit{aggregation module} such as those discussed in Sections \ref{sec:sc-gcn} and \ref{sec:sc-atn}. 
\begin{equation}\label{eq:agg}
    \bX^\ell \coloneqq \operatorname{AGG}\left(\{\bX_{\text{low},i}^\ell\}_{i=1}^{C_{\text{low}}}, \{\bX_{\text{band},i}^\ell\}_{i=1}^{C_{\text{band}}}\right).
\end{equation}
\\

\noindent\textbf{Graph Residual Convolution.}
After aggregating the outputs of the low-pass  channels 
and band-pass channels in Eq.~\ref{eq:agg}, we apply the \emph{graph residual convolution}, which acts as a low-pass filtering and aims to eliminate any high-frequency noise introduced by the scattering channels. This noise can arise, e.g., if there are various different label rates in different graph substructures. In this case, the scattering features may learn the difference between labeled and unlabeled nodes and  thereby produce high-frequency noise.

This filter uses a modified diffusion matrix given by 
$$
    \boldsymbol{A}_{\text{res}}(\alpha) = \frac{1}{\alpha+1} (\boldsymbol{I}_n+\alpha \boldsymbol{W} \boldsymbol{D}^{-1}),
$$
where the hyperparameter $\alpha$ determines the magnitude of the low-pass filtering. Choosing $\alpha = 0$ yields the identity (no filtering), while $\alpha \rightarrow \infty$ results in the random walk matrix $\boldsymbol{R} \coloneqq \boldsymbol{W} \boldsymbol{D}^{-1}$. Thus, $\bA_{\text{res}}$ can be interpreted as lying between a completely lazy random walk that never moves and a non-resting one $\boldsymbol{R}$ that moves at every time step.

The full residual convolution update rule is given by
\begin{equation*}
    \bX^{\ell+1} \coloneqq \bA_{\text{res}}(\alpha) \bX^\ell \bTheta_{\text{res}} + \bB_{\text{res}}.
\end{equation*}
The multiplication with  $\bTheta_{\text{res}}$ corresponds to a fully connected layer applied to the output from Eq.~\ref{eq:agg} (after filtering by $\bA_{\text{res}}$) with each neuron learning a linear combination of the signals output by the aggregation module.


\subsection{Scattering GCN}\label{sec:sc-gcn}

The Scattering GCN (Sc-GCN) network was first introduced in \cite{min2020scattering}. Here, the  aggregation module concatenates the filter responses horizontally yielding wide node representations $\bX^\ell$ of the form
\begin{equation}\label{eq:concat}
    [\bX_{\text{low},1}^\ell \mathbin\Vert \dots \mathbin\Vert \bX_{\text{low},C_{\text{low}}}^\ell \mathbin\Vert \bX_{\text{band},1}^\ell \mathbin\Vert \dots \mathbin\Vert \bX_{\text{band},C_{\text{band}}}^\ell].
\end{equation}
Sc-GCN then learns relationships between the channels via the graph residual convolution. \\

\begin{figure}[!t]
\centering
\includegraphics[width=\columnwidth]{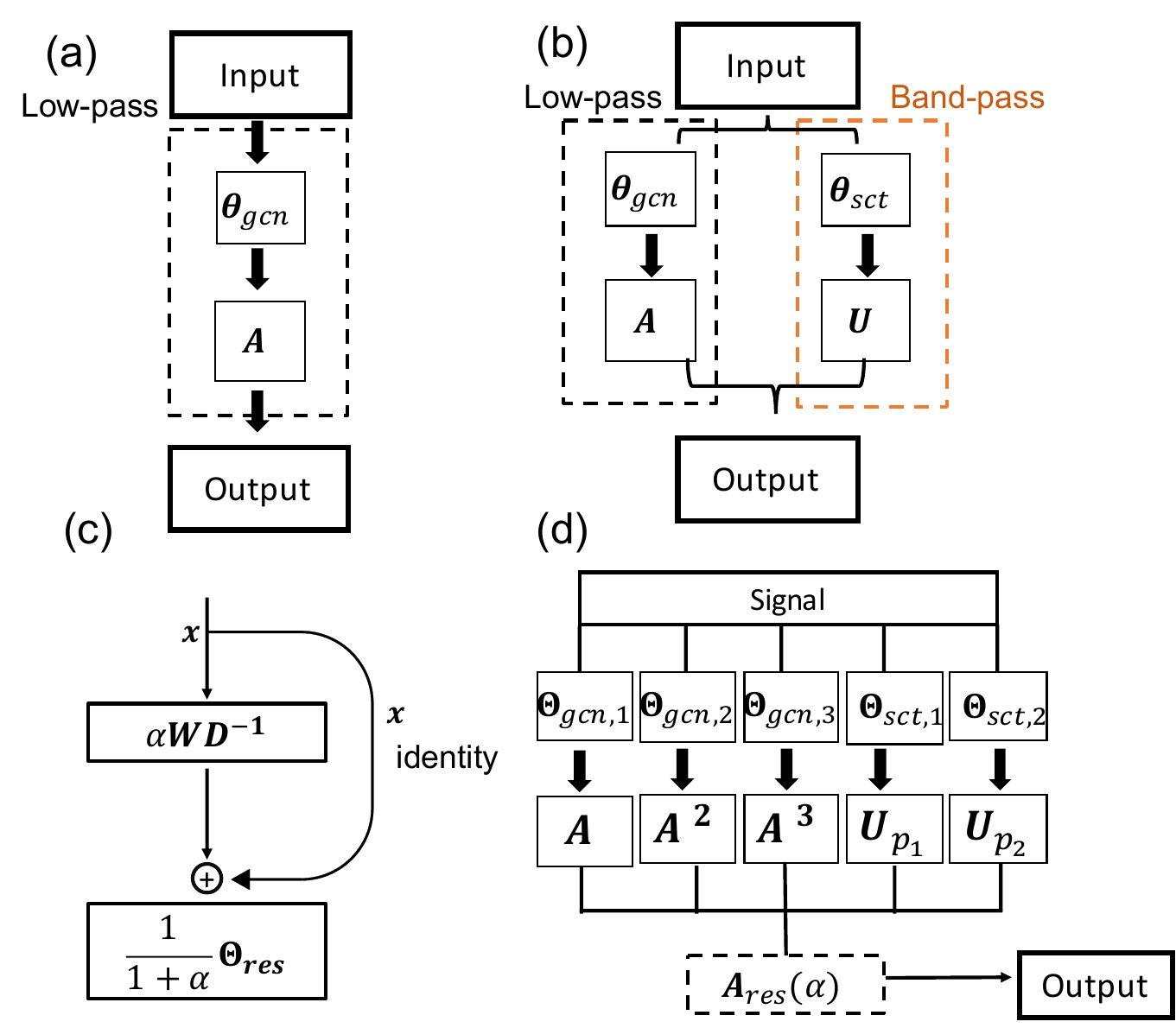}
\caption{(a) \& (b) comparison between GCN and our Sc-GCN: we add band-pass channels to collect different frequency components; (c) graph residual convolution layer; (d) Sc-GCN combines five network channels, followed by a graph residual convolution.}
\label{fig:Sc-GCN}
\end{figure}

\noindent\textbf{Configuration.}
The primary goal of Sc-GCN is to alleviate oversmoothing in popular semi-supervised node-level tasks on, e.g., citation networks. As regularity patterns in such datasets are often dominated by low-frequency information such as inter-cluster node-similarities, we choose our parameters to focus on low-frequency information. We use three low-pass filters, with  receptive fields of size (or radius) $r_k=1,2,3$,
and two band-pass filters.
We use $\sigma(\cdot) \coloneqq \vert\cdot\vert$ as our nonlinearity in all steps except the outermost nonlinearity. Inspired by the aggregation step in classical geometric scattering\cite{gao2019geometric}, for the outermost nonlinearity, we additionally apply the $q^{\text{th}}$ power at the node level, i.e., $\sigma(\cdot) \coloneqq \vert\cdot\vert^q$.

The paths $p$ and the parameter $\alpha$ from the graph residual convolution are tuned as hyperparameters of the network.

\subsection{Geometric Scattering Attention Network}\label{sec:sc-atn}

An important observation in Sc-GCN above is that the model decides globally about how to combine different channel information. The network first concatenates all of the  features from the low-pass and band-pass channels in Eq.~\ref{eq:concat} and then combines these features via multiplication with the weight matrix $\bTheta_{\text{res}}$. However, for complex tasks or datasets, important regularity patterns may vary significantly across different graph regions. In such settings, a model should ideally attend locally over the aggregation and adaptively combine filter responses at different nodes.

This observation inspired the design of the Geometric Scattering Attention Network (GSAN)~\cite{min2021geometric}. Drawing inspiration from ~\cite{velivckovic2017graph}, GSAN uses an aggregation module based on a multi-head node-level attention framework. However, the attention mechanism used in GSAN differs from \cite{velivckovic2017graph} by attending over the combination of the different filter responses rather than over the combination of node features from neighboring nodes.  
We will first focus on the processing performed independently by each attention head and then discuss a multi-head configuration.
\begin{figure}[!t]
\centering
\includegraphics[width=0.9\columnwidth]{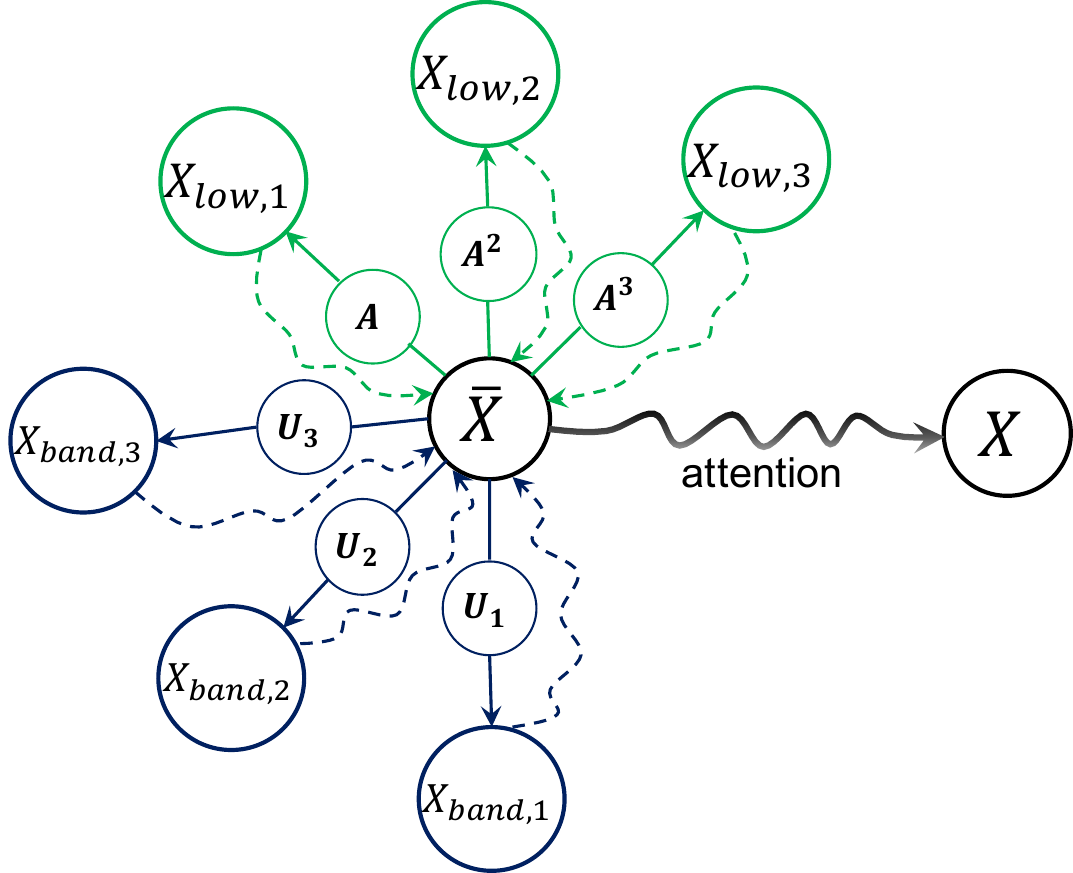}
\caption{Illustration of the proposed scattering attention layer.}
\label{fig:atten}
\end{figure}

    In a slight deviation from Sc-GCN, the weight matrices in Eq.~\ref{eq:gcn-mod} and Eq.~\ref{eq:scat-mod} are shared across the filters of each attention head. Therefore, both aggregation steps (Eq.~\ref{eq:gcn-node-b} and Eq.~\ref{eq:scat-node-b}) take the same input denoted by $\boldsymbol{\bar X}^\ell\coloneqq \bX^{\ell-1}\bTheta^\ell$. Next, we define $\boldsymbol{\bar{X}}_{\low,i}\coloneqq \bA^{r_i}\boldsymbol{\bar{X}}^\ell$ and $\boldsymbol{\bar{X}}_{\band,i}\coloneqq \vert \bU_{p_i}(\boldsymbol{\bar{X}}^\ell) \vert$ to be the outputs of the aggregation steps, Eq.~\ref{eq:gcn-node-b} and Eq.~\ref{eq:scat-node-b}, with the bias terms set to zero. 
We then compute \textit{attention score vectors} $\be_{\text{low},i}^{\ell}, \be_{\text{band},i}^{\ell} \in \R^n$ that will be used to determine the importance of each of the filter responses for every single node. We calculate
\begin{align*}
     \be_{\low,i}^{\ell} &= \operatorname{LeakyReLU} \left(\left[ \boldsymbol{\bar X}^{\ell} \Vert \boldsymbol{\bar{X}}_{\low,i}^{\ell} \right] \ba \right),
\end{align*}
with analogous $\be_{\text{band},i}^{\ell}$ and $\ba \in \R^{2 d_{\ell}}$ being a learned \textit{attention vector} that is shared across all filters of the attention head. These attention scores are then normalized across all filters using the softmax function. Specifically, we define 
\begin{align*}
     \balpha_{\low,j}^{\ell}
     &\coloneqq \frac{\exp(\be_{\low,j}^{\ell})}{\sum_{i=1}^{C_{\low}} \exp(\be_{\low,i}^{\ell})+\sum_{i=1}^{C_{\band}} \exp(\be_{\band,i}^{\ell})},
\end{align*}
where the exponential function is applied elementwise, and define  $\balpha_{\text{band},j}^{\ell}$ analogously.
Finally, for every node, the filter responses are summed together, weighted by the corresponding (node-to-filter) attention score.
We also normalize by the number of filters $C\coloneqq C_{\text{low}} + C_{\text{band}}$, which gives
\begin{align*}
     \bX^{\ell}
     = C^{-1} \tilde\sigma\bigg(& \sum_{j=1}^{C_{\text{low}}} \balpha_{\low,j}^{\ell} \odot \boldsymbol{\bar{X}}_{\low,j}^{\ell} + \sum_{j=1}^{C_{\band}} \balpha_{\band,j}^{\ell} \odot \boldsymbol{\bar{X}}_{\band,j}^{\ell} \bigg),
\end{align*}
where $\boldsymbol{\alpha} \odot \bX\coloneqq\diag(\boldsymbol{\alpha}) \bX$ denotes the Hadamard (elementwise) product of  $\boldsymbol{\alpha}$ with each column of $\bX$. We further use $\tilde\sigma(\cdot)=\text{ReLU}(\cdot)$ in the equation above.\\

\noindent\textbf{Multi-head Attention.}
Similar to other applications of attention mechanisms~\cite{velivckovic2017graph}, we use multi-head attention to stabilize the training, thus yielding as output of the aggregation module (by a slight abuse of notation)
\begin{equation}\label{eq:multi_head}
     \bX^{\ell} = \Vert_{\gamma=1}^\Gamma  \bX_\gamma^{\ell}\left( \bTheta_{\gamma}^{\ell}, \balpha_{\gamma}^{\ell}\right) ,
\end{equation}
concatenating $\Gamma$ attention heads. As explained above, each attention head has individually trained parameters $\bTheta_\gamma^\ell$ and $\balpha_\gamma^\ell$ and outputs a filter response $\bX_\gamma^\ell$. Similar to Sc-GCN, this concatenation results in wide node representations that are further refined by the graph residual convolution. \\

\noindent\textbf{Configuration.}
For GSAN, we set $C_{\text{low}} = C_{\text{band}} = 3$, giving the model balanced access to both low-pass and band-pass information. The aggregation process of the attention layer is shown in Fig.~\ref{fig:atten}, where $\boldsymbol{U}_{1,2,3}$ represent three first-order scattering transformations with $\boldsymbol{U}_1 \boldsymbol{x} \coloneqq  \boldsymbol{\Psi}_1 x$,  $\boldsymbol{U}_2 \boldsymbol{x} \coloneqq \boldsymbol{\Psi}_2 x$ and $\boldsymbol{U}_3 \boldsymbol{x} \coloneqq \boldsymbol{\Psi}_3 x$. The number of attention heads $\Gamma$ and the parameter $\alpha$ from the graph residual convolution are tuned as hyperparameters of the network.

\section{Theory of Hybrid Scattering Networks}\label{sec:theory}



In this section, we compare the descriptive power of the scattering transform to 
a class of GNNs defined below, which we refer to as aggregate-combine GNNs (AC-GNNs). 
We note that in practice, AC-GNNs are typically implemented with relatively narrow receptive fields in order to avoid oversmoothing. However, this leads to the underreaching problem, as discussed in Sections~\ref{sec:intro} and~\ref{sec:challenges}. The scattering transform, on the other hand, is able to utilize a large receptive field while avoiding oversmoothing. Therefore, it is able to simultaneously overcome both the oversmoothing and underreaching problems.

The family of AC-GNNs includes many popular GNN architectures such as GCN~\cite{kipf2016semi}, GIN~\cite{xu2018powerful} and GraphSAGE~\cite{hamilton2017inductive}. However, importantly, we note that it does not include the scattering transform (because AC-GNNs only use local  -- one-step neighborhood -- aggregations). We also note that this definition of an AC-GNN is quite similar to analogous notions used in, e.g., \cite{barcelo2020logical,xu2018powerful}.

\begin{defn}[Aggregate-Combine GNN]\label{def:ac-gnn}
    An \textit{aggregate-combine GNN} (AC-GNN) is a network, where the update rule for each layer is defined on the node level according to
    \begin{equation}\label{eq:ac-gnn}
        F_{\ac}(\bX)_v \coloneqq \sigma\left(\com \left( f(\bX_v, d_v), \agg_{w\in\nbh_{v}} g(\bX_w, d_v, d_w) \right)\right),
    \end{equation}
    for arbitrary functions $f: \R^{1\times d} \times \N \rightarrow \R^{1\times d'}$, \newline $g: \R^{1\times d} \times \N^2 \rightarrow \R^{1\times d'}$, $\com: \R^{1\times d'}\times \R^{1\times d'} \rightarrow \R^{1\times d'}$, $\sigma:\R \rightarrow \R$ an (optional) activation function applied elementwise, and $\agg$ any function mapping a multi-set of vectors in $\R^{1\times d'}$ (i.e., $\ldcb \ba_i^\top \in \R^{1\times d'}\rdcb_{i\in\mathcal{I}}$) to a vector in $\R^{1\times d'}$. 
\end{defn}


As discussed in Section \ref{sec:gcn}, the aggregations used in many popular AC-GNNs effectively act as localized averaging operators and focus primarily on low-frequency information. Therefore, deep AC-GNNs have the undesirable property of oversmoothing the input features if too many layers are used. As a result, most AC-GNNs typically only use two or three layers in order to avoid the oversmoothing problem. To understand this at an intuitive level, we follow the lead of \cite{wang2020unifying} and consider a simplified version of GCN in which $\sigma$ is the identity and the weight matrix is given by $\bTheta=\frac{1}{2}\Id_n$ and also consider the network without the renormalization trick. In this case, the filter used in GCN is given by $\frac{1}{2}\left(\Id_n + \bD^{-1/2} \bW \bD^{-1/2}\right)$ and so we have $\bhg[i]=1-\lambda_i/2$ in Eq.~\ref{eq:gsp-conv}. This implies that if $\bx$ is \emph{any} vector which is orthogonal to the lead eigenvector, we have that the representation of $\bx$ after $\ell$ layers satisfies
\begin{equation*}
    \left\|\frac{1}{2}\left(\Id_n + \bD^{-1/2} \bW \bD^{-1/2}\right)^\ell\bx\right\|_2\leq \left(1-\frac{\lambda_2}{2}\right)^\ell\|\bx\|_2.
\end{equation*}
    Therefore, the output of a deep (simplified) GCN will converge exponentially fast to the projection of $\bx$ onto the bottom eigenspace, which is why GCNs with many layers typically achieve poor performance.
    
By contrast, the use of band-pass filters allows to incorporate a larger receptive field without essentially projecting the input features onto the bottom eigenspace. Consider for example the following result %
from \cite{perlmutter2019understanding}. 
\begin{prop}[Proposition 2.2 of \cite{perlmutter2019understanding}]\label{prop:frame}
    The wavelet filter bank $\mathcal{W}=  \{\bPsi_k, \bPhi_K\}_{k=0}^K$ introduced in Eq.~\ref{eqn: W} is a non-expansive frame with respect to the weighted norm defined by $\|\bx\|_{\bD^{-1/2}}\coloneqq\sum_{i=1}^n|\bx[i]|^2 d_{v_i}^{-1}$ whose lower-frame bound is a universal constant independent of $J$ and the graph topology. That is, there exists a universal constant $c_1>0$ such that
\begin{align*}
    c_1\|\bx\|^2_{\bD^{-1/2}}\leq  \sum_{k=0}^K &\|\bPsi_k\bx\|^2_{\bD^{-1/2}} \\  + &\|\bPhi_K\bx\|^2_{\bD^{-1/2}}\leq \|\bx\|^2_{\bD^{-1/2}}.
\end{align*}
\end{prop}
The fact that the lower bound $c_1$ does not depend on $J$ is important because the receptive field of an $L$-layer geometric scattering transform is $2^JL$.
Therefore, if we choose $J$ to be sufficiently large, the scattering transform as formulated in \cite{perlmutter2019understanding} is able to incorporate long-range interaction in the network while still preserving a substantial portion of the input signal energy.
Our construction differs slightly from \cite{perlmutter2019understanding} in that our wavelet filter bank includes only the $\bPsi_k,$ but not $\bPhi_K$. However, one may derive a result similar to Proposition \ref{prop:frame}, but where the lower bound $c_1$ depends on $\lambda_2$ (but still does not depend on $J$). We refer the reader to the proof of Proposition 4.1 of \cite{gama2018diffusion} for details.\footnote{While \cite{gama2018diffusion} does not consider weighted norms, this technicality can be readily handled  by imitating the proof of Proposition 2.2 of \cite{perlmutter2019understanding}.}.
Unlike the scattering transform, in order to avoid oversmoothing, most AC-GNNs typically use networks with small receptive fields.  While this \textit{does} help avoid the problem of oversmoothing, it unfortunately creates the problem of underreaching. In the remainder of this section, we will focus on demonstrating that this underreaching diminishes the power of the network to discriminate different nodes in certain situations. 

We will now begin our analysis of \textit{node discriminability}\footnote{Formal definitions are provided in Appendix.~\ref{sec:apx-def}, \label{apx-def}}, i.e., the ability of a network to produce different representations of two nodes $v$ and $w$ in its hidden layers in the context of the underreaching problem. 
We will let  $\nbh_{\underline{v}}^{K}$ denote the \textit{$K$-step node neighborhood}\textsuperscript{\ref{apx-def}} of a node $v$ 
(including $v$ itself), and for $V_S\subset V$ we will let  $G(V_S)$ denote the corresponding \textit{induced subgraph}.\textsuperscript{\ref{apx-def}} We will say two induced subgraphs $G(V_S)$ and $G(V_{S'})$ are \textit{isomorphic}\textsuperscript{\ref{apx-def}} if they have identical geometric structure and write
$G(V_S) \cong^\phi G(V_{S'})$
to indicate that $\phi$ is an isometry (geometry preserving map) between them.
In the definition below, we introduce a class of \textit{graph-intrinsic} node features that encode graph topology information.
We will use these features as the input  for GNN models in order to produce geometrically informed node representations.



\begin{defn}[Intrinsic Node Features]\label{def:intrinsic}
A nonzero node feature matrix $\bX$ is \textit{$K$-intrinsic} if for any $v,w\in V$ such that $G\left(\nbh_{\underline{v}}^{K}\right)$ is isomorphic to $G\left(\nbh_{\underline{w}}^{K}\right)$, we have $\bX[v]=\bX[w]$.
\end{defn}
These intrinsic node features encode important $K$-local geometric information in the sense that if $\bX[v]\neq \bX[w]$, then the $K$-step neighborhoods of $v$ and $w$ must have different geometric structure. To further understand this definition, we observe that the degree vector, or any (elementwise) function of it, is a one-intrinsic node feature matrix. Setting $\bX[v]$ to the average node degree of nodes in $\nbh_{\underline{v}}^{K-1}$ yields $K$-intrinsic node features. As a slightly more complicated example, features with $\bX[v]$ the number of triangles contained in $\nbh_{\underline{v}}^K$ are also $K$-intrinsic. Fig.~\ref{fig:intrinsic} illustrates how such node features can help distinguish nodes from different graph structures.

\begin{figure}[!t]
    \centering
    \subfloat[]{
        \includegraphics[width=0.315\columnwidth]{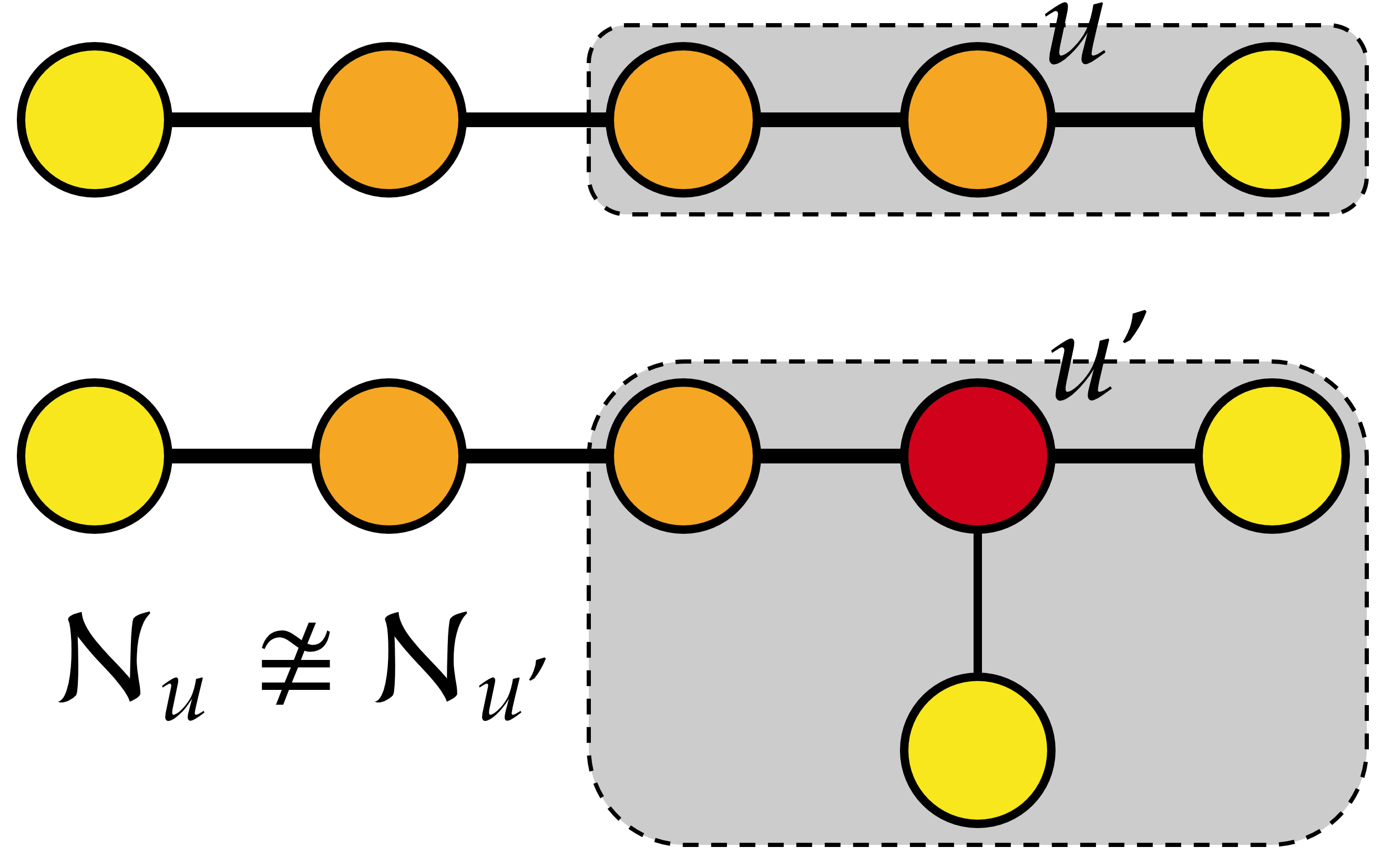}
        \label{subfig:feat a}
    }
    \subfloat[]{
        \includegraphics[width=0.315\columnwidth]{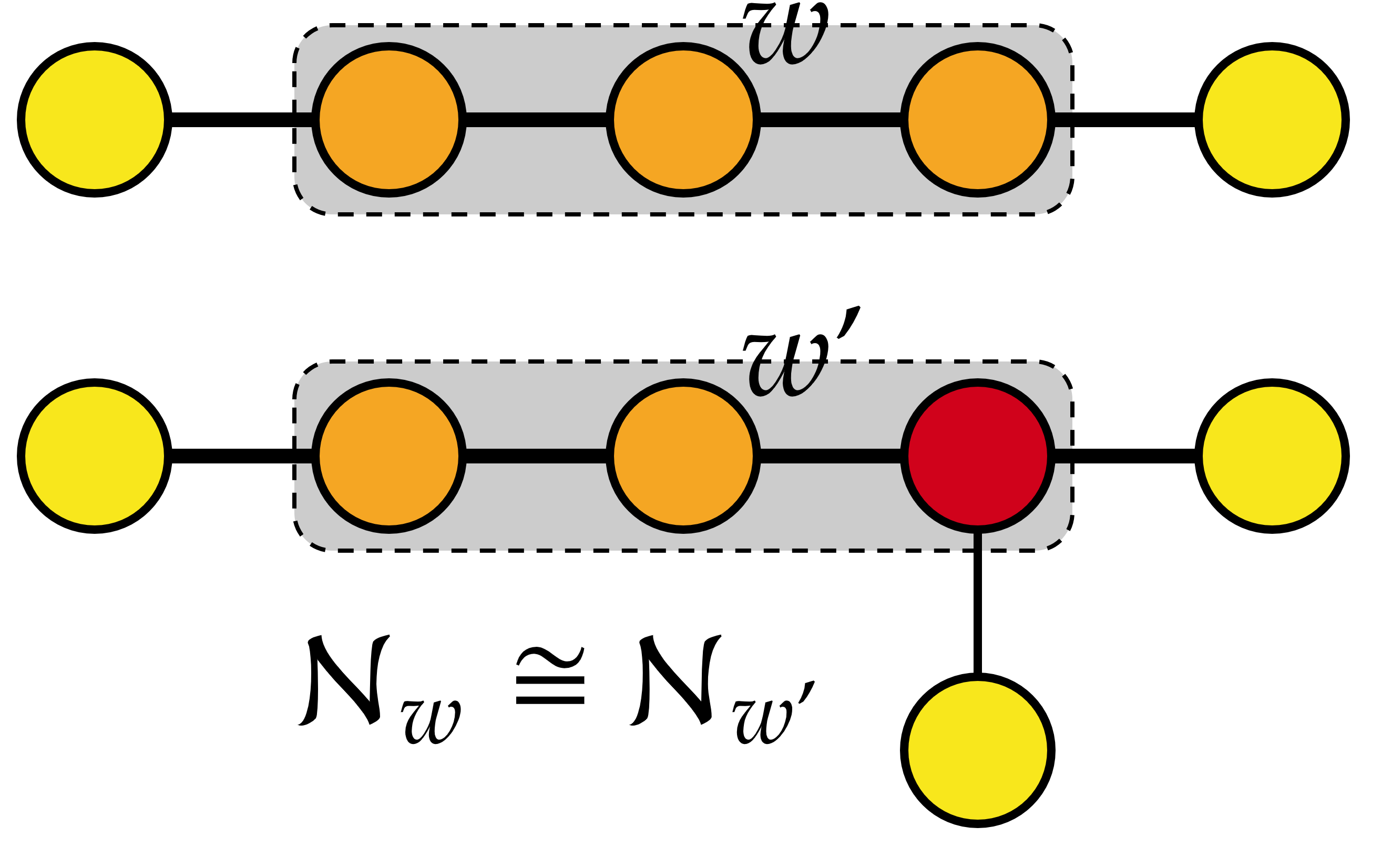}
        \label{subfig:feat b}
    }
    \subfloat[]{
      \includegraphics[width=0.315\columnwidth]{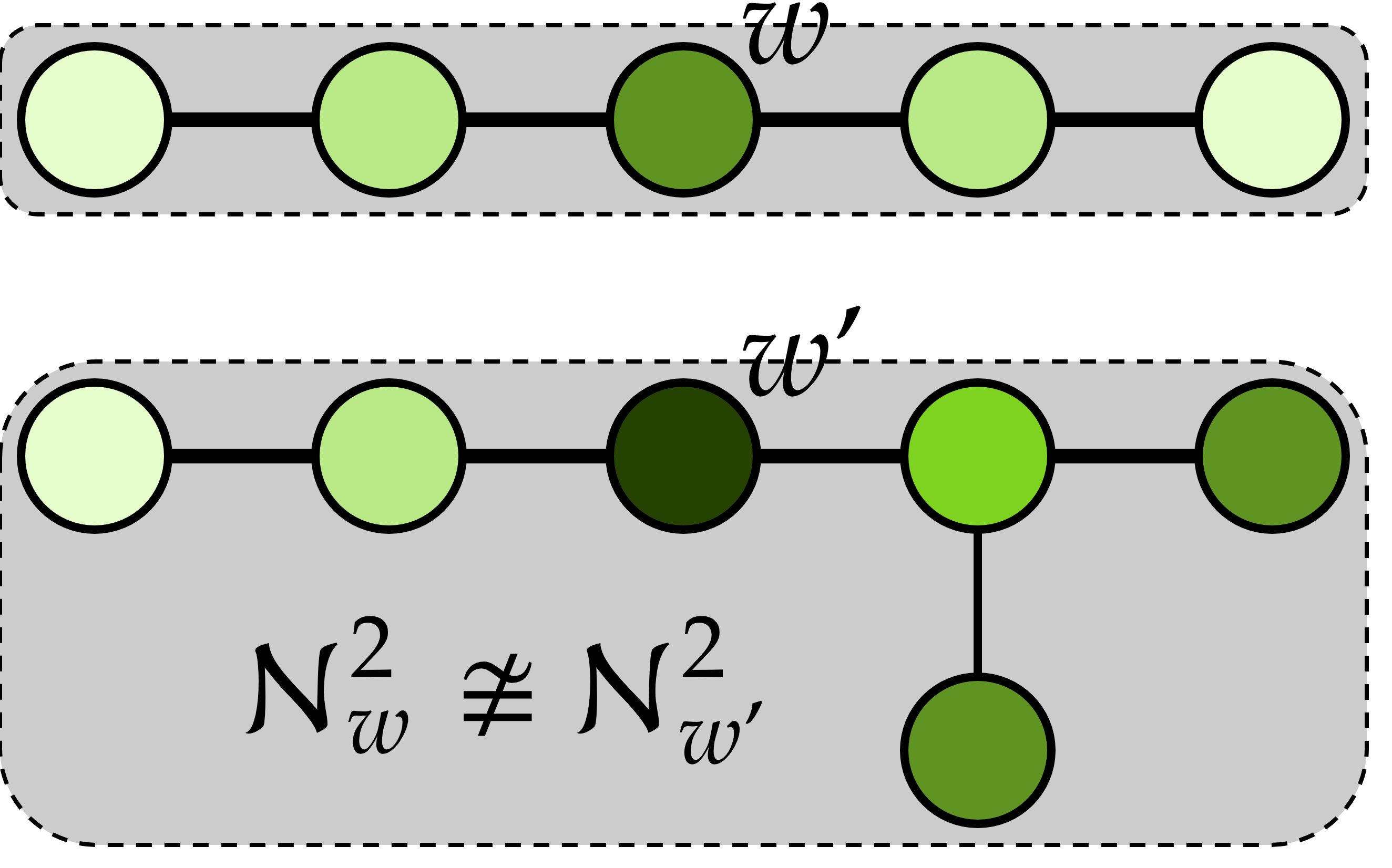}
      \label{subfig:feat c}
    }
    \caption{Intrinsic node features for two graphs (top/bottom). We compare nodes from the isomorphic horizontal path that is part of both graphs. Color coding indicates different node feature values and differs in (c). Node features are one-intrinsic (degrees) in (a) and (b), and two-intrinsic (average degree in one-step neighborhoods) in (c). Gray area represents subgraph used for the feature assignment at the indicated node.}
    \label{fig:intrinsic}
\end{figure}

As a concrete example, consider the task of predicting traffic on a network of Wikipedia articles, such as the Chameleon dataset~\cite{rozemberczki2021multi}, with nodes and edges corresponding to articles and hyperlinks.
Here, intrinsic node features provide insights into the topological context of each node.
A high degree suggests a widely cited article, which is likely to have a lot of traffic, while counting $k$-cliques or considering average degrees from nearby nodes can shed light on the density of the local graph region.

The following theorem characterizes situations where AC-GNNs cannot discriminate nodes based on the information represented in intrinsic node features. 
 It can be seen as a formal formulation of the underreaching problem~\cite{barcelo2020logical}. We note that an analogous result can also be established for graph attention networks, e.g., GAT~\cite{velivckovic2017graph} and GATv2~\cite{brody2021attentive}.


\begin{thm}\label{thm:main-1-new}
    Let $\phi : V \rightarrow V$, $v\in V$, and $K,L\in\mathbb{N}$ such that $G\left(\nbh_{\underline{v}}^{K+L}\right) \cong^\phi G\Big(\nbh_{\underline{\phi(v)}}^{K+L}\Big)$. Let $\bX^0$ be any $K$-intrinsic node feature matrix and let $\bX^L$ be the output of an $L$-layer AC-GNN, i.e., $F_{\ac}^L\circ\dots\circ F_{\ac}^1$, with each layer as in Eq.~\ref{eq:ac-gnn}.
    Then,  $\bX^L[v]=\bX^L[\phi(v)]$, and so $v$ and $\phi(v)$ cannot be discriminated based on the filtered node features $\bX^L$ from an $L$-layer AC-GNN.
\end{thm}
    Theorem~\ref{thm:main-1-new} is proved by using an induction argument to show that 
    $\bX^\ell_u=\bX^\ell_{\phi(u)}$ for all $u\in\nbh_{\underline{v}}^{L-\ell}$
    and all $0\leq \ell\leq L$. 
%
Full details can be found in the appendix~\ref{sec:apx-main-1}. 

Next, we introduce a notion of structural difference that allows us to analyze situations where networks with wide receptive fields such as scattering networks have strictly greater discriminatory power than networks with smaller receptive fields. 




\begin{defn}[Structural Difference]\label{def:struc diff}
    For two $\phi$-isomorphic induced subgraphs $G(V_S)\cong^\phi G(V_{S'})$, a \textit{structural difference} relative to $\phi$ and a node feature matrix $\bX$ is manifested at $u\in V_S$ if $\bX[u]\neq\bX[\phi(u)]$
    or $d_u\neq d_{\phi(u)}$.\footnote{Note that $d_u\neq d_{\phi(u)}$ is only relevant for nodes in the \textit{boundary}\textsuperscript{\ref{apx-def}} $\partial S$ because $d_w=d_{\phi(w)}$ for all \textit{interior}\textsuperscript{\ref{apx-def}} nodes $w\in\interior(S)$, as the degree is 1-intrinsic. For $u\in\partial S$, we further assume $d_{\phi(u)}\bX[u]\neq d_u\bX[\phi(u)]$.}
\end{defn}

Notably, if  the $K$-step neighborhoods of two nodes $v$ and $\phi(v)$ are $\phi$-isomorphic and $\bX$ is any $K$-intrinsic node feature matrix, then no structural difference relative to $\phi$ and $\bX$ can be manifested at $v$. 
Theorem~\ref{thm:main-2-new} stated below shows that a scattering network will be able to produce different representations of two nodes $v$ and $\phi(v)$ if there are structural differences in their surrounding neighborhoods, assuming (among a few other mild assumptions) that a certain pathological case is avoided.
The following notation and definition characterize this pathological situation that arises, roughly speaking, when two sums $\sum_i a_i$ and $\sum_i b_i$ coincide, even though $a_i$ and $b_i$ do not coincide. We note that although this condition seems complicated, it is highly unlikely to occur in practice.



\begin{notation}\label{notation-new}
    Let $\phi : V \rightarrow V$, $v \in V$, $D \in \mathbb{N}$, such that $G\left(\nbh_{\underline{v}}^{D}\right) \cong^\phi G\left(\nbh_{\underline{\phi(v)}}^{D}\right)$. 
    Assume that there exists at least one node in $\nbh_{\underline{v}}^{D}$ where a structural difference is manifested rel. to $\phi$ and node features $\bX$.
    We define
    $$
        V_{\text{diff}} \coloneqq \Big\{u\in \nbh_{\underline{v}}^D : \text{ struc. difference rel. to $\phi$ and $\bX$ at } u \Big\}.
    $$
    and we fix the following notations:
    \begin{enumerate}
        \item[(i)] Let $d\coloneqq\min\{d(u,v) : u\in V_{\text{diff}}\}$ and define the node set $V_{\text{diff}}^{d}\coloneqq\{u\in V_{\text{diff}} : d(u,v)=d\}$. Note that these are exactly the nodes in $\nbh_{v}^{d}$ where a structural difference is manifested relative to $\phi$ and $\bX^0$.
        \item[(ii)] Let $\tau\in\mathbb{N}$ be the number of paths of minimal length between $v$ and nodes in $V_{\text{diff}}^{d}$, and denote these by  $p^{(i)}\coloneqq \left(u_0^i,\dots,u_d^i=v\right)$ with $u_0^i\in V_{\text{diff}}^{d}$, $1\leq i\leq\tau$.
        \item[(iii)] Let $U_j\coloneqq \cup_{i=1}^\tau\{u_j^{i}\}$,
        and define the \textit{generalized path} from $V_{\text{diff}}^d$ to $v$ as $P\coloneqq (U_0,U_1,\dots,U_d)$.
        \item[(iv)] For $u\in U_0$, define $\delta_u^0 \coloneqq \bX^0[u] - \bX^0[\phi(u)]$. \\
        For $u\in U_j$, $0\leq j=\leq d$ and $\bY^j \coloneqq \bP^j \bX^0$, define
        $$
            \textstyle \delta_u^j \coloneqq \frac{1}{2} \sum_{w\in \nbh_u \cap U_{j-1}} d_w^{-1} \bY^{j-1}_w - d_{\phi(w)}^{-1} \bY^{j-1}_{\phi(w)}
        $$
        and set
    \end{enumerate}
    $$
        \Delta_j \coloneqq \{ u\in U_j : \, \exists \, w\in\nbh_{u}\cap U_{j-1} \, \text{such that} \, \delta_w^{j-1}\neq 0 \}.
    $$
\end{notation}

\begin{defn}[No Coincidental Correspondence]\label{def:nocc}
Let $\nbh_{\underline{v}}^{d}$ and $\Delta_j$ as in Notation~\ref{notation-new}.
We say that $\interior\left(\nbh_{\underline{v}}^{d}\right)$ exhibits \textit{no coincidental correspondence} (no-cc) if in each non-empty $\Delta_j$, $1\leq j \leq d$, there exists at least one $u$ such that $\delta_u^j \neq 0$.
\end{defn}
\begin{thm}\label{thm:main-2-new}
    As in Theorem~\ref{thm:main-1-new}, let $\phi : V \rightarrow V$, $v\in V$, and $K,L\in\N$ such that $G\left(\nbh_{\underline{v}}^{K+L}\right) \cong^\phi G\Big(\nbh_{\underline{\phi(v)}}^{K+L}\Big)$, and consider any K-intrinsic node feature matrix $\bX$. Further assume that there exists at least one structural difference rel. to $\phi$ and $\bX$ in $\nbh_{\underline{v}}^{K+L}$, and let $1\leq d\leq K+L$ be the smallest positive integer such that a structural difference is manifested in $\nbh_{\underline{v}}^{d}$. If the nonlinearity $\sigma$ is strictly monotonic and $\interior\left(\nbh_{\underline{v}}^{d}\right)$ exhibits no coincidental correspondence rel. to $\phi$ and $\bX$, then one can define a scattering configuration $(p,\bTheta)$ such that scattering features $F_{p\psct}(\bX)$ defined as in Eq.~\ref{eq:scat-node} discriminate $v$ and $\phi(v)$.
\end{thm}


Theorem~\ref{thm:main-2-new} provides a large class of situations where scattering filters are able to discriminate between two nodes even if networks with smaller receptive fields are not. In particular, even if the $K+L$ step neighborhood of $v$ and $\phi(v)$ are isomorphic, it is possible for there to exist a $u$ in this $K+L$ step neighborhood such that a $K$-intrinsic node feature takes different values at $u$ and $\phi(u)$. Theorem~\ref{thm:main-2-new} shows that scattering will produce differing representations of  $v$ and $\phi(v)$, while Theorem~\ref{thm:main-1-new} shows that an $L$-layer AC-GNNs like GCN will not. In the remarks below we will  discuss minor variations of Theorem~\ref{thm:main-2-new}.
In the appendix~\ref{sec:apx-mod-main-2}, we will also present a modified version of Theorem~\ref{thm:main-2-new}, where the assumption of no coincidental correspondence is replaced by an assumption on the (generalized) path between the structural differences and the investigated node.
\begin{rem}
    The absolute value operation is not monotonic. 
    Therefore, Theorem~\ref{thm:main-2-new} cannot be directly applied in this case.
    However, the above result and all subsequent results can be extended to the  case $\sigma=|\cdot|$ as long as certain pathological cases are avoided. Namely, Theorem \ref{thm:main-2-new} will remain valid as long as the features at $u$ are assumed not to be the negatives of the features at $\phi(u)$, i.e., $\bX[u]\neq-\bX[\phi(u)]$ for nodes $u\in\interior\left(\nbh_{\underline{v}}^{d}\right)$, and similarly we must modify  Definition~\ref{def:nocc} to avoid, for example, the case where $\sum_{w\in \nbh_u \cap U_{j-1}} d_w^{-1} \bX^{j-1}[w] =-\sum_{w\in \nbh_u \cap U_{j-1}} d_{\phi(w)}^{-1} \bX^{j-1}[\phi(w)]$. We note that while this condition is complex, it is rarely violated  in practice.
\end{rem}

\begin{rem}
    A result analogous to  Theorem \ref{thm:main-2-new} is also valid for any permutation of the three steps Eq.~\ref{eq:scat-node-a}-\ref{eq:scat-node-c} (even with added bias terms as in Eq.~\ref{eq:scat-mod}). In particular, it applies to the update rule
    $
        \bX^{\ell} = \sigma \left( \bU_p \big(\bX^{\ell - 1}\big) \bTheta^{\ell} + \bB^{\ell} \right).
    $
\end{rem}


We will provide a sketch of the proof of  Theorem~\ref{thm:main-2-new} here and provide details in the appendix. The key to the proof will be applying Lemma \ref{thm:lem-onion-new} stated below. 

\begin{lem}\label{thm:lem-onion-new}
    Let $\phi : V \rightarrow V$, $\tv \in V$, and $\tD \in \N$ such that $G\left(\nbh_{\underline{\tv}}^{\tD}\right) \cong^\phi G\left(\nbh_{\underline{\phi(\tv)}}^{\tD}\right)$.
    Assume there exists at least one node in $\nbh_{\underline{\tv}}^{\tD}$ where a structural difference is manifested relative to $\phi$ and 
    to $\btX$, and let $1\leq \td\leq\tD$ be the smallest positive integer such that a structural difference is manifested in $\nbh_{\underline{\tv}}^{\td}$.
    Assume that $\interior\left(\nbh_{\underline{\tv}}^{\td}\right)$ exhibits no coincidental correspondence relative to $\phi$ and $\btX$,
    and let $\tP\coloneqq \left(\tU_0,\tU_1,\dots,\tU_{\td}\right)$ be the generalized path as defined in Notation~\ref{notation-new}.
    Then, no structural difference is manifested in $\nbh_{\underline{\tv}}^{\td-j} \setminus \tU_j$, relative to $\phi$ and filtered node features $\bY^j\coloneqq \bP^j\btX$ and there exists at least one node in $u\in\tU_j$ and where a structural difference is manifested and $\delta^j_u\neq 0$.
\end{lem}


Below, we provide a sketch of the main ideas of the proof. A complete proof is provided in Appendix~\ref{sec:apx-lem-onion}.

\begin{proof}[Proof Sketch]
We use induction to show that at every step the matrix $\bP$ propagates the information about the structural difference one node layer closer towards $\tv$.
    
    By Notation~\ref{notation-new}(i), we have $\widetilde{U}_0=V_{\text{diff}}^{\tilde{d}}$. Therefore, the case $j=0$ is immediate. In the inductive step, it suffices to show
    $$
        \bY^{j+1}[u]\neq\bY^{j+1}[\phi(u)]
    $$
    and $\delta_u^{j+1} \neq 0$ for at least one $u\in \tU_{j+1}$ and
    $$
        \bY^{j+1}[w]=\bY^{j+1}[\phi(w)]
    $$
    for all $w\in\nbh_{\underline{\tv}}^{\td-(j+1)}\setminus \tU_{j+1}$ under the assumption that these results are already valid for $\bY^j$.
    This can be established by writing $
    \bP^{j+1}\btX = \bP\bP^{j}\btX = \bP \bY^j$ and using the inductive hypothesis together with the definition of $\bP$.
\end{proof}

    

We now use Lemma \ref{thm:lem-onion-new} to sketch the proof of Theorem~\ref{thm:main-2-new}. The complete proof is provided in the appendix~\ref{sec:apx-main-2}.

\begin{proof}[Proof Sketch for Theorem~\ref{thm:main-2-new}]
    We need to show that we can choose the parameters in a way that guarantees $F_{p\text{-}sct}(\bX)[v]\neq F_{p\text{-}sct}(\bX)[\phi(v)]$. For simplicity, we set $\bTheta=\Id_n$. 
    In this case, since $\sigma$ is strictly monotonic, and therefore injective, it suffices to show that we can construct $p$ with 
    \begin{equation}\label{eqn: Us are different-sketch-new}
        \bU_p(\bX)[v]\neq \bU_p(\bX)[\phi(v)].
    \end{equation}
    Using binary expansion, we may choose $k_1,\ldots,k_m\in \mathbb{N}_0$, $k_i<k_{i+1}$, such that $d=2^{k_1}+ \ldots + 2^{k_m}$
    and set $p\coloneqq(k_1,\ldots,k_m)$. Given $p$, we define truncated paths  $p_{:i}\coloneqq (k_1,\ldots,k_i)$
    and let $t_i\coloneqq\sum_{j=1}^i 2^{k_j}$ for $1\leq i \leq m.$ We also let 
    $p_{:0}$  be the empty path of length 0 and set $t_0=0.$ 
    
    Recalling the generalized path $P=(U_0,\ldots,U_d)$ defined in Notation \ref{notation-new},
    we will use induction to show that for at least one node in $U_{t_i}$ a structural difference is manifested, while no structural difference manifests in $\mathcal{N}_{\underline{v}}^{d-t_i} \setminus U_{t_i}$ rel. to $\phi$ and $\bZ^i\coloneqq \bU_{p_{:i}}\bX$ for $0\leq i \leq m$.
    Since $t_m=d$ and $p_{:m}=p$,   this will imply Eq. \ref{eqn: Us are different-sketch-new} and thus prove the theorem. As  with Lemma~\ref{thm:lem-onion-new}, the base case $i=0$ follows from Notation~\ref{notation-new}(i).
    In the inductive step, it suffices to show that 
    \begin{equation}\label{eq:wav-step-diff2}
        \bZ^{i+1}[u]\neq\bZ^{i+1}[\phi(u)],
     \end{equation}
    for at least one $u\in U_{t_{i+1}}$ and
    \begin{equation}\label{eq:wav-step-eq2}
        \bZ^{i+1}[w]=\bZ^{i+1}[\phi(w)],
    \end{equation}
    for all $w\in\nbh_{\underline{v}}^{d-t_{i+1}}\setminus U_{t_{i+1}}$, under the assumption that  these results are true for $i$. Since $\bZ^{i+1}=\bPsi_{k_{i+1}} \sigma\big(\bZ^i\big)$ and $\sigma$ is increasing and therefore injective, it suffices to show that these results are preserved under multiplication by a  wavelet matrix $\bPsi_{k_{i+1}}=\bP^{2^{k_{i+1}-1}} - \bP^{2^{k_{i+1}}}$. This can be verified using Lemma~\ref{thm:lem-onion-new} (with $\sigma(\bZ^i)$ in place of $\btX$) and the inductive hypothesis.
\end{proof}

\begin{rem}
In the proof of Theorem~\ref{thm:main-2-new}, for the sake of concreteness, we chose $\bTheta=\Id_n$ and $p$ such that $|p|=\sum_{j=1}^m2^{k_i}=d$. However, inspecting the proof we see that our analysis may also be extended to paths with $|p|\geq d$ and also that 
the choice of $\boldsymbol{\Theta}$ does not matter as long as $\bTheta$ is invertible. In particular, if the entries of $\bTheta$ are generated i.i.d.\ at random from a continuous probability distribution, the conclusions of the theorem will hold with probability one\cite{Rudelson_invertibilityof}. We also note that Theorems \ref{thm:main-1-new} and \ref{thm:main-2-new} may be extended to the case where the update rule features a bias term $\bB$ as in Eq.~\ref{eq:gcn-mod} and \ref{eq:scat-mod} since $\bX\rightarrow\bX+\bB$ is injective.
\end{rem}

Results analogous to Theorem~\ref{thm:main-2-new} can likely also be proven for deep AC-GNNs with $d$ subsequent layers. The proof would closely follow the idea of Lemma~\ref{thm:lem-onion-new} and would likely be slightly simpler than the proof of Theorem~\ref{thm:main-2-new}.
Indeed, inspecting the proof of Theorem~\ref{thm:main-2-new}, we see that part of the reason we need both Eq.~\ref{eq:wav-step-diff2} and Eq.~\ref{eq:wav-step-eq2} is because the definition of $\bPsi_k$ involves subtraction. Therefore, establishing Eq.~\ref{eq:wav-step-eq2} could be skipped for an AC-GNN.
However, we would need considerably stronger assumptions when working with a $d$-layer AC-GNN $F_\ac^d \circ \dots \circ F_\ac^1$. Apart from an assumption analogous to no coincidental correspondence, we would additionally need injectivity of all functions $\sigma_i, \com_i, g_i$ that characterize $F_\ac^i, 1\leq i\leq d$, following Definition~\ref{def:ac-gnn}.
Moreover, scattering filters are a more practical choice for two major reasons. Firstly, Scattering filters nicely balance the trade-off between oversmoothing and underreaching and are able to utilize a broad receptive field while still preserving information contained in higher frequencies as explained in Proposition \ref{prop:frame} and the subsequent discussion. 
Secondly, scattering filters exhibit wide receptive fields with significantly fewer trained parameters. 

\section{Empirical Results}\label{sec:results}
We now present our empirical results starting with semi-supervised node classification. We first show that our Sc-GCN model~\cite{min2020scattering} is able to overcome the oversmoothing and underreaching problems and exhibits superior performance to either a pure GCN Network or a pure scattering network. We then show that we can further improve performance, particularly on complex datasets, by using GSAN\cite{min2021geometric}, which incorporates an attention mechanism. Finally, we also show that our models perform well on graph-level tasks.

\subsection{Semi-Supervised Node Classification}

\noindent\textbf{Scattering GCN.}
We compare the performance of Sc-GCN to both the  original GCN \cite{kipf2016semi} as well as the closely related ChebNet \cite{defferrard2016convolutional} and a network which uses Gaussian random fields to propagate label information~\cite{zhu2003semi}. Additionally, we compare against \cite{zou2020graph}, which uses a pure-scattering approach. Further, we also compare Sc-GCN to two methods designed, in part, to address the oversmoothing problem. In~\cite{li2018deeper}, this problem is directly addressed by using partially absorbing random walks \cite{wu2012learning} to slow the mixing of node features in regions of the graph with high connectivity.
The GAT graph attention network from~\cite{velivckovic2018graph}, on the other hand, addresses this problem indirectly via an attention mechanism which trains an adaptive node-wise weighting of the smoothing operation. Lastly, we compare to an SVM classifier  which ignores the graph structure. 

In our experiments, we used four popular datasets summarized in Tab.~\ref{table:dataset}. For full details on these datasets, please see, e.g.,~\cite{yang2016revisiting} for Citeseer, Cora, and Pubmed and~\cite{pan2016tri} for DBLP. As discussed in~\cite{li2018deeper}, the oversmoothing problem is most acute in highly connected networks where information mixes rapidly and nodes quickly become indistinguishable. Therefore, we included networks with different sizes and connectivity structures and order the datasets by connectivity structure (with the most connected on the right). Consistent with our expectation, Tab.~\ref{tab_test accuracies} and Fig.~\ref{fig_results} show that the advantages of our network are most significant for the networks with high connectivity.

We implemented all of these methods based upon the code made available by the original authors, and, for fair comparison, we tuned and evaluated each of these methods using the standard splits provided for benchmark datasets. Whenever possible, we made sure that our results are consistent with previously reported accuracies. We tuned our method, both the hyperparameters and number of scattering and GCN channels, via a grid search (over the same set for all datasets), and we used the same cross-validation procedure for our method and the competing methods. In Appendix \ref{sec:apx-empirical}, we provide further details and also perform an ablation study evaluating the importance of each component in our proposed architecture.

\begin{table}[!t]
\caption{Dataset characteristics: number of nodes, edges, and features; mean $\pm$ std.\ of node degrees; ratio of \#edges to \#nodes.}
\label{table:dataset}
\centering
\begin{tabular}{cccccccc}
\hline
Dataset & Nodes & Edges& Features & Degrees  & $\frac{\text{Edges}}{\text{Nodes}}$\\
\hline
Citeseer & 3,327 & 4,732 & 3,703 & 3.77$\pm$3.38 & 1.42 \\
Cora & 2,708 & 5,429 & 1,433 & 4.90$\pm$5.22 & 2.00 \\
Pubmed & 19,717 & 44,338 & 500 & 5.50$\pm$7.43 & 2.25 \\
DBLP & 17,716 & 52,867 & 1639 & 6.97$\pm$9.35 & 2.98 \\
\hline
\end{tabular}
\end{table}

\begin{table}[!t]
\caption{Classification accuracy (top two marked in bold; best one underlined) of Scattering GCN on four benchmark datasets (ordered by increasing connectivity) compared to four other GNNs~\cite{velivckovic2018graph,li2018deeper,kipf2016semi,defferrard2016convolutional}, a non-GNN approach \cite{zhu2003semi} based on belief propagation, a pure graph scattering baseline~\cite{zou2020graph}, and a nongeometric baseline only using node features with linear SVM.}
\label{tab_test accuracies}
\centering
\begin{tabular}{lcccr}
\hline
Model & Citeseer & Cora & Pubmed & DBLP \\
\hline
Sc-GCN (ours) & \textbf{71.7} & \underline{\textbf{84.2}} & \underline{\textbf{79.4}}& \underline{\textbf{81.5}} \\
GAT~\cite{velivckovic2018graph} & \underline{\textbf{72.5}} & \textbf{83.0} & 79.0 & 66.1 \\
Partially absorbing~\cite{li2018deeper} & 71.2 & 81.7 & \textbf{79.2} & 56.9 \\
GCN~\cite{kipf2016semi} & 70.3 & 81.5 & 79.0 & 72.0\\
Chebyshev~\cite{defferrard2016convolutional} & 69.8 & 78.1 & 74.4 & 57.3 \\
Label Propagation~\cite{zhu2003semi} & 58.2 & 77.3 & 71.0 & 53.0 \\
\hline
Graph scattering~\cite{zou2020graph} &  67.5 & 81.9 & 69.8 & \textbf{69.4} \\
\hline
Node features (SVM) &  61.1 & 58.0 & 49.9 & 48.2 \\

\hline
\end{tabular}
\end{table}

We report test classification accuracy in Tab.~\ref{tab_test accuracies}, which shows that our approach outperforms other methods on three out of the four considered datasets. Only on Citeseer, we are slightly outperformed by GAT. However, this is the dataset with the weakest connectivity structure (see Tab.~\ref{table:dataset}), while at the same time having the most informative node features (achieving 61.1\% accuracy via linear SVM without considering graph information). In contrast, on DBLP, which exhibits the richest connectivity structure and least informative features (only 48.2\% SVM accuracy), we significantly outperform GAT (by over 15\% improvement). Notably, GAT itself significantly outperforms all other methods (by 6.8\% or more) apart from the graph scattering baseline~\cite{zou2020graph}.

\begin{figure*}[!t]
    \centering
    \includegraphics[width=0.25\textwidth]{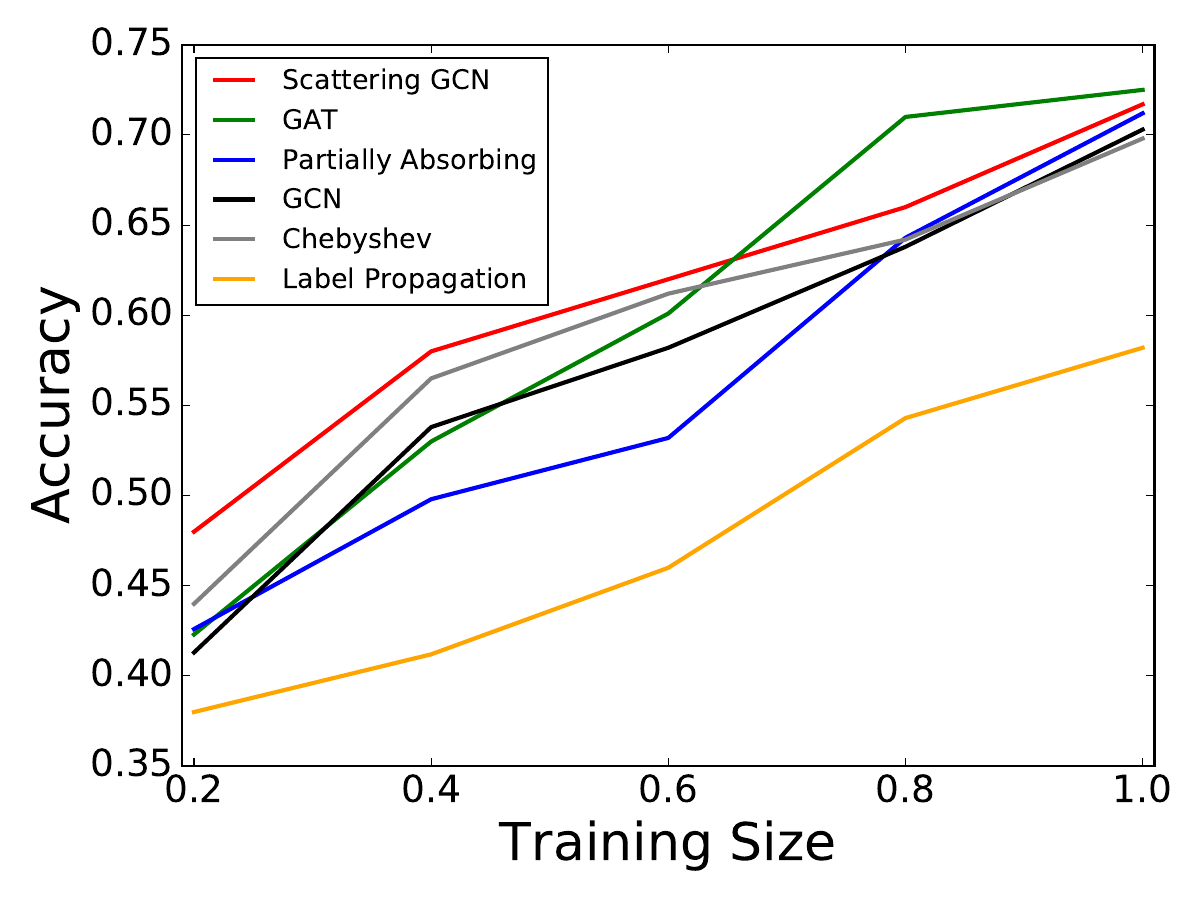}%
    \includegraphics[width=0.25\textwidth]{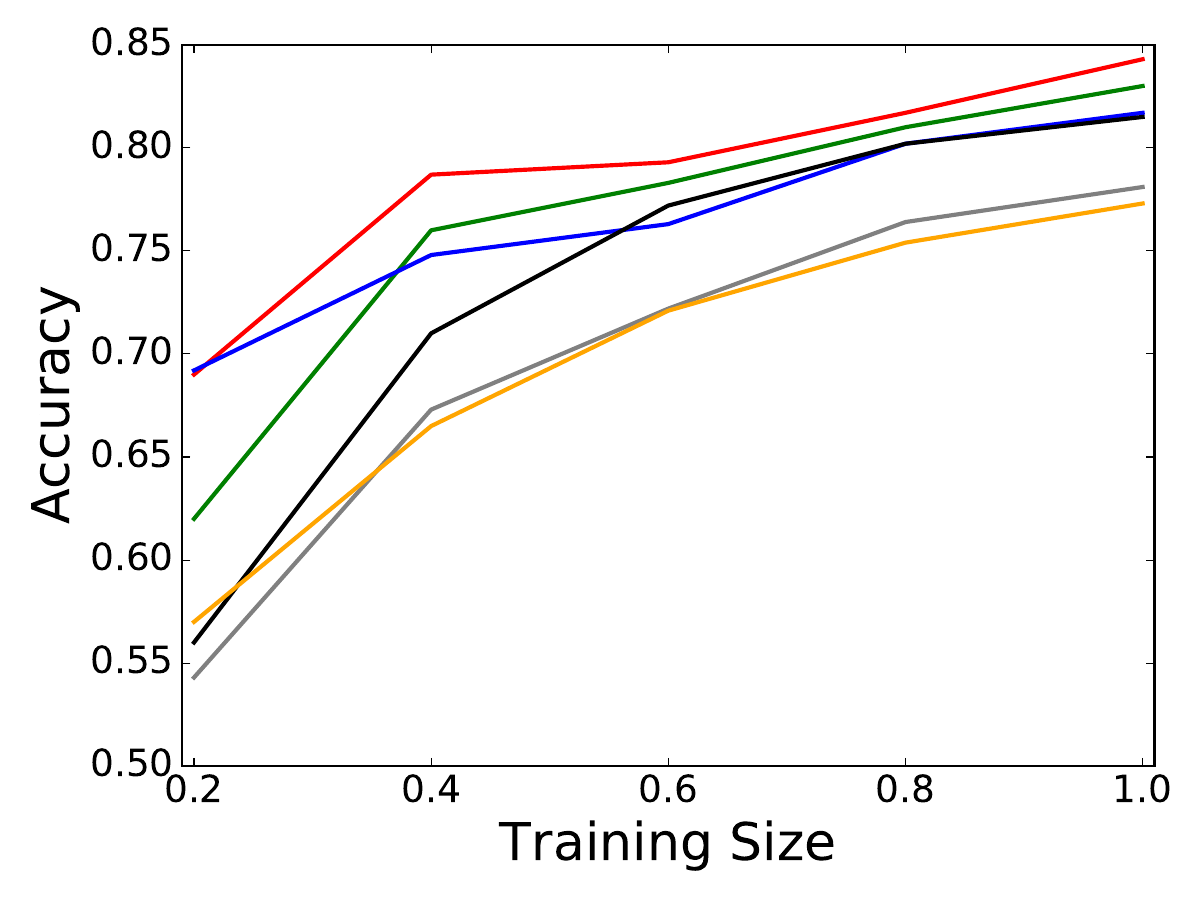}%
    \includegraphics[width=0.25\textwidth]{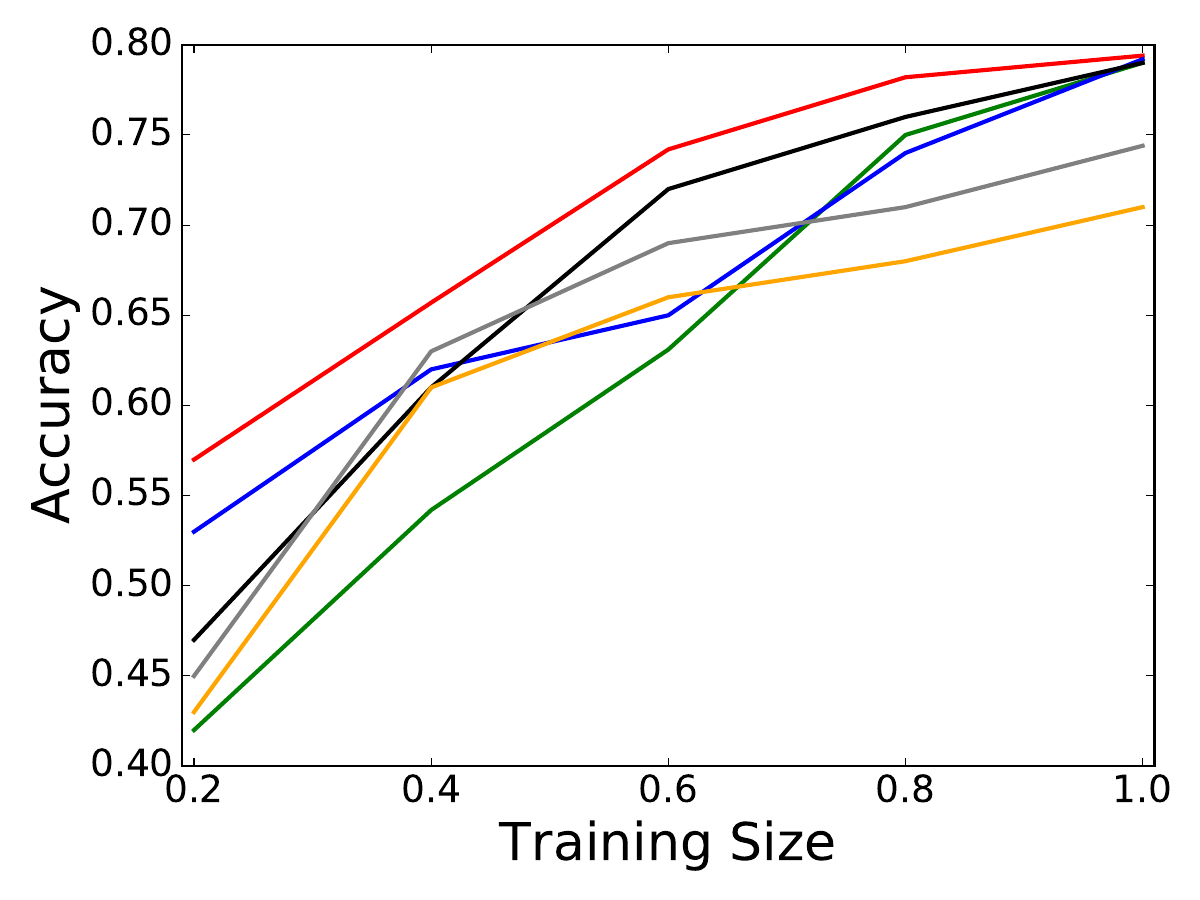}%
    \includegraphics[width=0.25\textwidth]{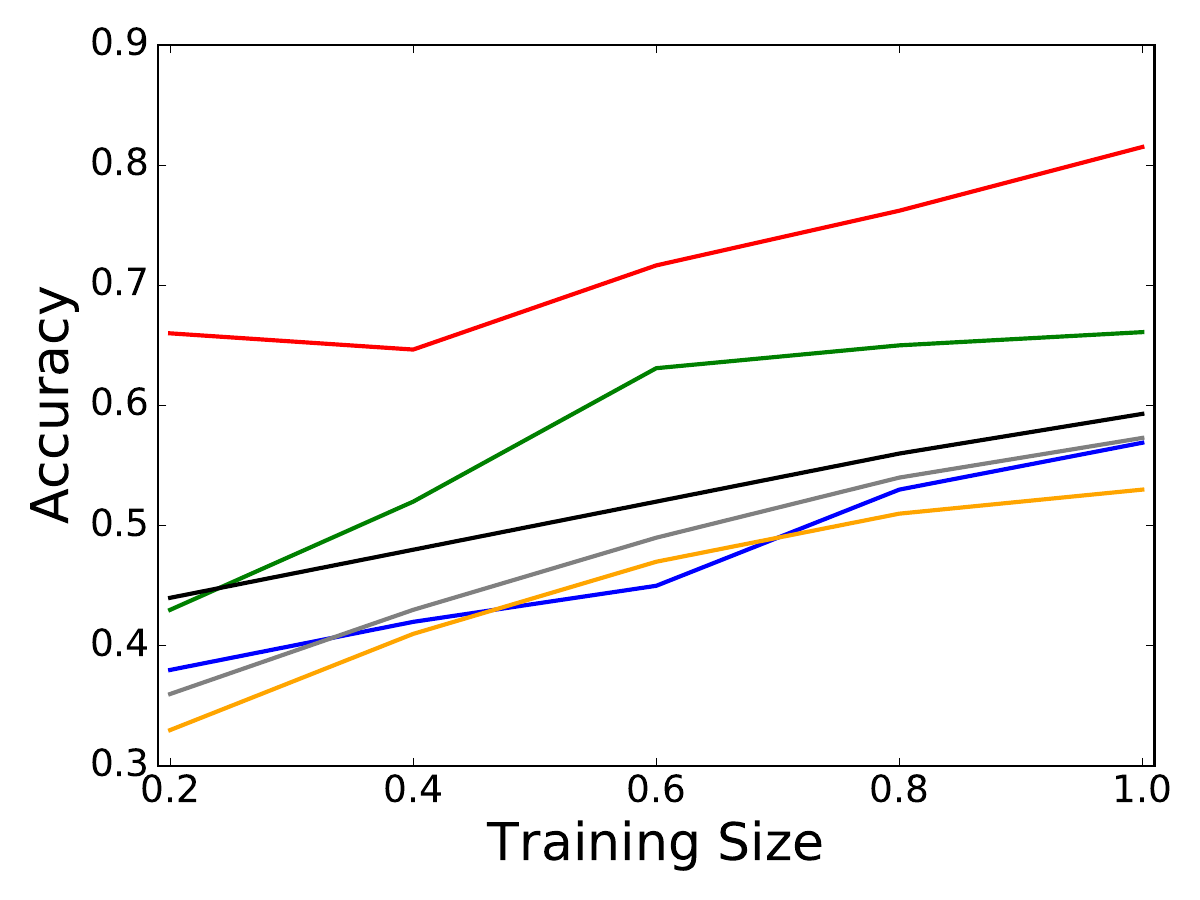}%
    \hfil
    \subfloat[Citeseer]{\includegraphics[width=0.25\textwidth]{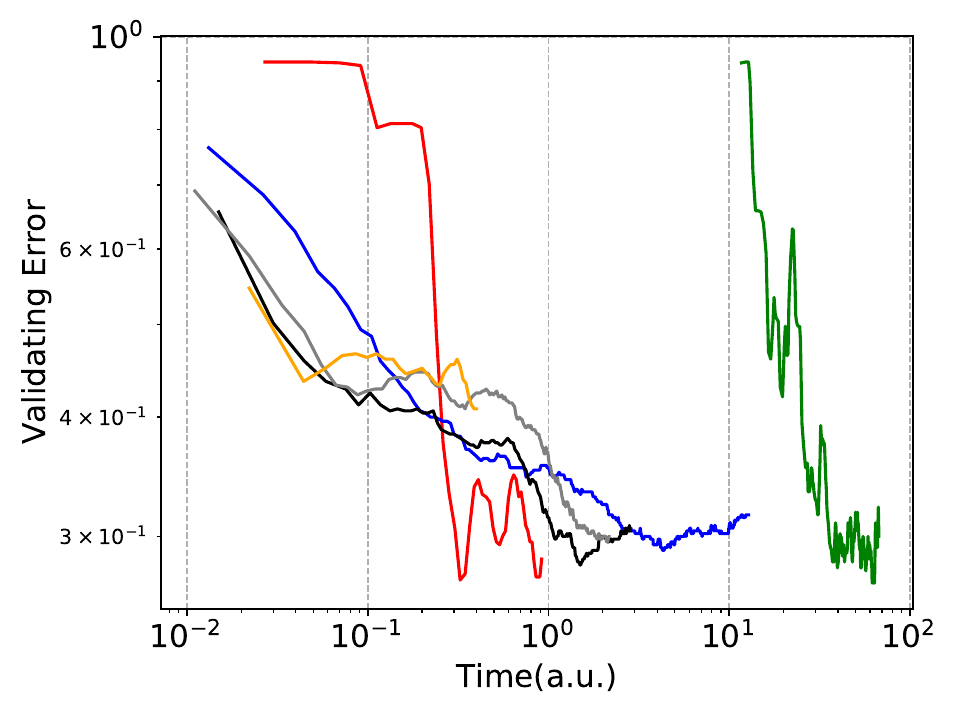}%
    \label{fig:citeseer}}
    \subfloat[Cora]{\includegraphics[width=0.25\textwidth]{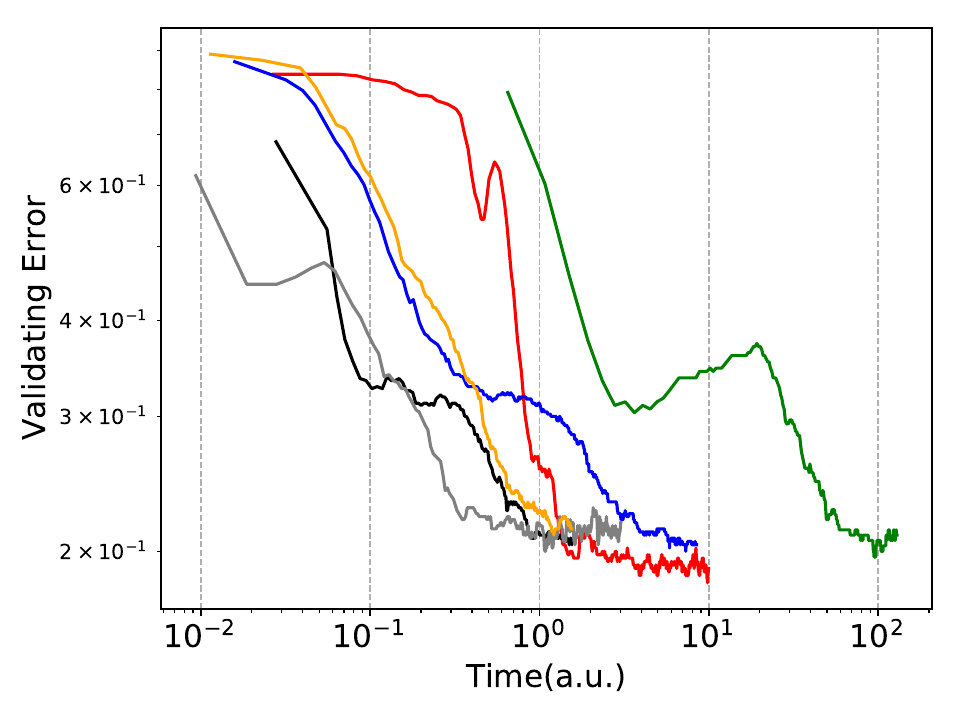}%
    \label{fig:cora}}
    \subfloat[Pubmed]{\includegraphics[width=0.25\textwidth]{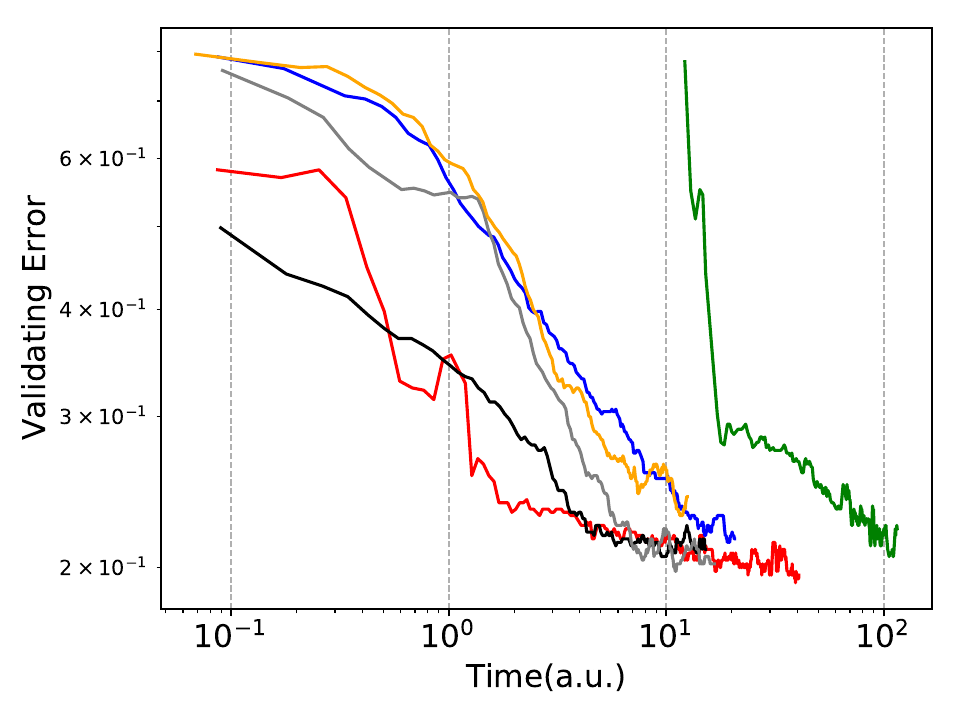}%
    \label{fig:pubmed}}
    \subfloat[DBLP]{\includegraphics[width=0.25\textwidth]{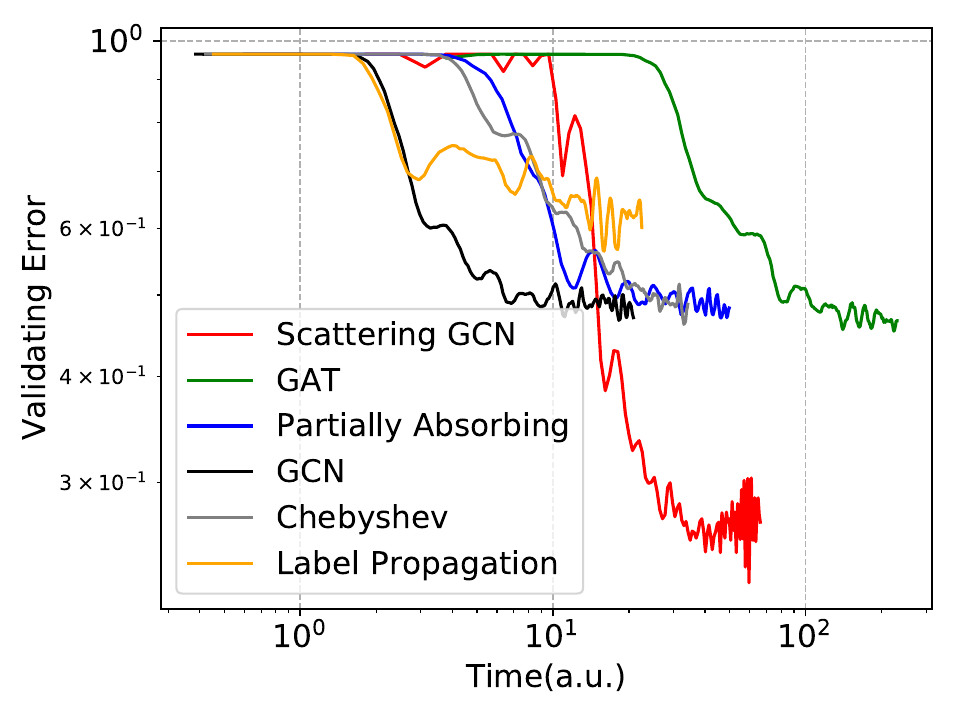}%
    \label{fig:dblp}}
    \caption{Impact of training set size (top) and training time (bottom) on classification accuracy and error (correspondingly); training size measured relative to the original training size of each dataset; training time and validation error plotted in logarithmic scale; runtime measured for all methods on the same hardware, using original implementations accompanying their publications.}
    \label{fig_results}
\end{figure*}

The task of semi-supervised learning is motivated by the fact that in real-world applications one typically only has labels for a small fraction of the nodes. Therefore, we next analyze the performance of our network with limited amounts of training data. Fig.~\ref{fig_results} (top) plots the training accuracy of Sc-GCN, as well as the other networks (except the baselines), as the amount of training data  varies from 20\% to 100\% of its original amount. Across all networks, we see that Sc-GCN exhibits the best performance when the training data is sufficiently limited. Moreover, we see that it outperforms GAT, which is the next best performing network in terms of overall accuracy (see Tab.~\ref{tab_test accuracies}), on all datasets whenever the training data is reduced to 60\% of its original size or less.

In Fig.~\ref{fig_results} (bottom), we plot the evolution of the validation error during the training process. We note that Sc-GCN is able to reach a low training error significantly faster than GAT, which is the next most competitive in terms of overall accuracy. Several other methods do achieve low validation errors faster on several datasets, but their final accuracy is much lower than that of Sc-GCN. Moreover, on Pubmed, which is the largest dataset, Sc-GCN achieves a low validation error at about the same rate as GCN and at a significantly faster rate than all other methods. \\

\begin{table*}[!b]
\renewcommand{\arraystretch}{1.3}
\caption{Dataset characteristics \& comparison of node classification test accuracy. Datasets are ordered by increasing homophily.}\label{tab:data_accu}
\centering
\begin{tabular}{|c||c|c|c|c||
c|c||c|c|}
\hline
Dataset  & Classes & Nodes & Edges & Homophily & GCN & GAT & Sc-GCN & GSAN \\\hline
Texas & 5 & 183 &295 & 0.11 & 59.5 &58.4 & 60.3 & 60.5\\ \hline
Chameleon & 5 & 2,277 & 31,421 &0.23 & 28.2 & 42.9 & 51.2 & 61.2 \\ \hline
CoraFull & 70 & 19,793 &63,421 &0.57 & 62.2 & 51.9 & 62.5 & 64.5 \\ \hline
Wiki-CS & 10 & 11,701 &216,123 &0.65 &77.2 &77.7 & 78.1 & 78.6\\ \hline
Citeseer & 6 & 3,327 &4,676 &0.74 &70.3 &72.5 & 71.7 & 71.3\\ \hline
Pubmed & 3 & 19,717 &44,327 &0.80 &79.0 &79.0 & 79.4 & 79.8\\ \hline
Cora & 7 & 2,708 &5,276 &0.81 &81.5 &83.0 & 84.2 & 84.0\\ \hline
DBLP & 4 & 17,716 &52,867 &0.83 &72.0 &66.1 & 81.5 & 84.3\\ \hline
\end{tabular}
\end{table*}

\noindent\textbf{Geometric Scattering Attention Network.}
Having established the practical utility of the hybrid network Sc-GCN above, we next show that performance may be further improved by using GSAN, which builds upon Sc-GCN by incorporating an attention mechanism.
We evaluate performance again for semi-supervised node classification and compare to two popular models (namely GCN~\cite{kipf2016semi} and GAT~\cite{velivckovic2018graph}), as well as the original Sc-GCN. We extend our experiments to eight benchmark datasets that we order according to increasing homophily, as shown in Tab.~\ref{tab:data_accu}. The homophily of a graph is quantified by the fraction of edges that connect nodes that share the same label. High homophily indicates a relatively smooth label distribution, while low homophily suggests a more complex labeling pattern. Texas and Chameleon are low-homophily datasets where nodes correspond to webpages and edges to links between them, with classes corresponding to webpage topic or monthly traffic (discretized into five levels), respectively~\cite{pei2020geom}. Wiki-CS consists of nodes that represent computer science articles with the edges representing the hyperlinks~\cite{mernyei2020wiki}. CoraFull is the larger version of the Cora dataset~\cite{bojchevski2018deep}.
The remaining four datasets (i.e., Cora, Citeseer, Pubmed, DBLP) are the citation networks that we already studied when evaluating Sc-GCN above. Notably, compared to the added datasets, the latter four exhibit relatively high homophily.

We split the datasets into train, validation and testing sets. The hyperparameters (number of attention heads $\Gamma$, residual parameter $\alpha$ and channel widths) are tuned via grid search using the validation set.

The results in Tab.~\ref{tab:data_accu} show that both Sc-GCN and GSAN outperform GCN and GAT on seven out of eight datasets. The advantages of GSAN over Sc-GCN are particularly notable on the medium-size and large datasets Chameleon and CoraFull that exhibit relatively low-homophily. The most striking result is for Chameleon. Here, Sc-GCN clearly outperforms GCN and GAT (by at least 8.2\%), but GSAN considerably improves performance even more (by additional 10\%).

Analyzing the node-wise attention weights (see Sec.~\ref{sec:sc-atn}) allows us to understand the importance of different channels for different datasets. We consider here the ratio between attention assigned to band-pass and low-pass channels (over all attention heads). For that, we first sum up attention over low-pass and band-pass channels, respectively, i.e., $\balpha_\low\coloneqq\sum_\gamma\sum_i \boldsymbol{\alpha}_{\low,i}^\gamma$ and $\balpha_\band\coloneqq\sum_\gamma\sum_i \boldsymbol{\alpha}_{\band,i}^\gamma$. Then, for each node $v$, we calculate the ratio $\zeta_v\coloneqq\balpha_\band[v] / \balpha_\low[v]$. In Fig.~\ref{fig_attention-correlation}, we present the distributions of $\{\zeta_v : v\in V\}$ for four datasets.

\begin{figure}[!t]
    \centering
    \includegraphics[width=\columnwidth]{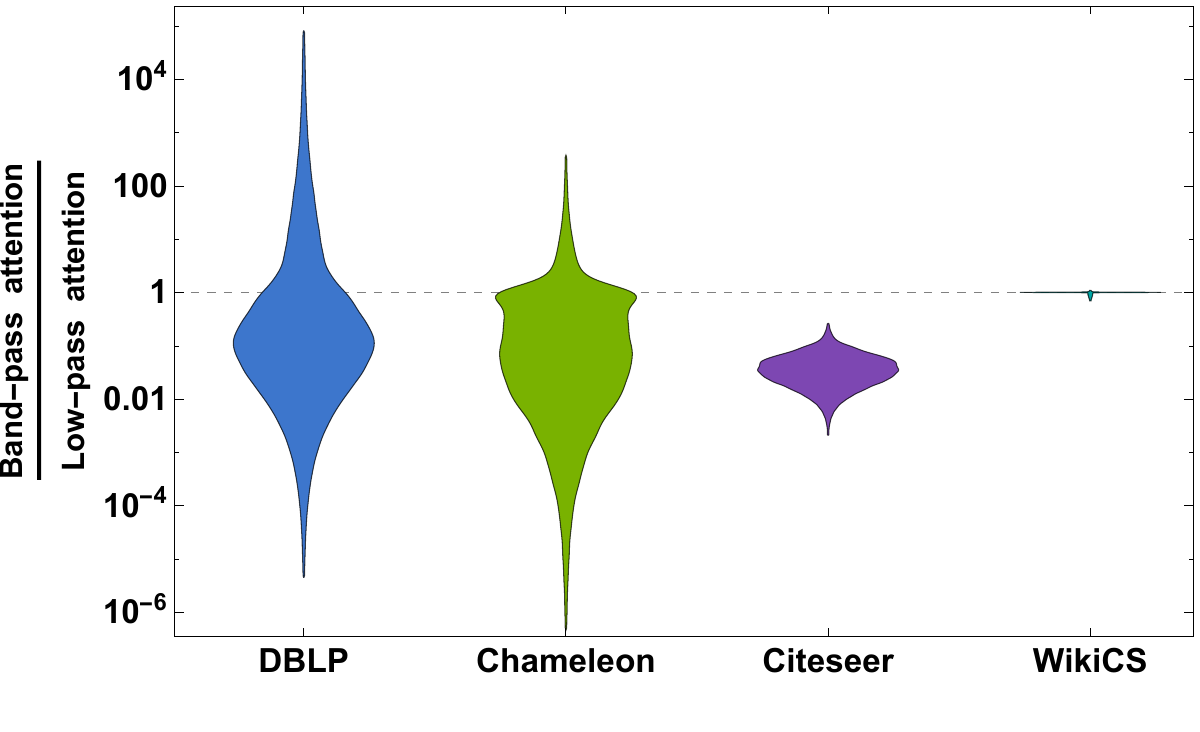}
    \vspace{-15pt}
    \caption{Distribution of attention ratios per node between band-pass (scattering) and low-pass (GCN) channels across all heads for DBLP, Chameleon, Citeseer, and WikiCS.}
    \label{fig_attention-correlation}
\end{figure}

Examining the distributions, we observe four different dataset profiles. Citeseer and WikiCS exhibit minimal spread of attention ratios. In Citeseer, most of the attention is given to the low-pass channels, while WikiCS shows very balanced channel usage. In contrast, DBLP and Chameleon exhibit large spreads. Although the majority of nodes in both datasets values low-pass filters more, some nodes in Wiki-CS pay up to five times the attention to band-pass channels compared to low-pass ones (and vice versa). Interestingly, the distribution for Chameleon has two peaks, suggesting two node populations that require different band-information, which we interpret to be a driving factor for the improvement of GSAN on Chameleon (see Tab.~\ref{tab:data_accu}).

\subsection{Graph Classification and Regression}

\noindent\textbf{Graph Classification.}
We use COLLAB  and IMBD-BINARY as benchmark datasets.
These are ego networks extracted from scientific collaborations and movie collaborations, respectively \cite{yanardag2015deep}. COLLAB contains 5,000 graphs with three classes and IMBD-BINARY contains 1,000 graphs with two classes. We use an 80-10-10 split between training, validation and test. The datasets contain graph structures but no associated graph signals. In our work, we compute the eccentricity (for connected graphs), degree and clustering coefficient of each vertex, and use these as input signals to our network. 

We use a one-layer GSAN  with 8 heads and 16 hidden units for COLLAB and one-layer GSAN with 4 heads and 16 hidden units for IMBD-BINARY, followed by one graph residual convolutional layer.  We then  apply a Set2Set module \cite{vinyals2015order} followed by a multi-layer perceptron (MLP). Set2Set is a graph pooling architecture\cite{hochreiter1997long} where we apply an LSTM-based module for $k$ processing steps iteratively on the graph signal. During each processing step, node-wise feature signals with variable graph size are integrated using an LSTM-based attention mechanism. In our model, the GSAN-based layer outputs a graph signal $\bX\in\R^{n\times d}$ and the Set2Set layer outputs a graph level feature vector $Set2Set(\bX) \in \R^{2d}$ for each graph. After the Set2Set layer, we apply an MLP layer $\R^{2d}\rightarrow \R$ on the graph level signal for the classification task.  We use a Set2Set layer with 3 processing steps and a two-layer MLP for graph pooling.
We also provide a pure-scattering baseline (Sc-only) based on the approach proposed in \cite{zou2020graph} together with an SVM classifier.
Our results in Tab.~\ref{tab:class} show that both proposed hybrid scattering networks show improvement over the compared baselines.

\begin{table}[!ht]
\renewcommand{\arraystretch}{1.3}
\caption{Comparison of graph classification test accuracy (higher is better).}\label{tab:class}
\centering
\begin{tabular}{|c|c|c|c|c|c|}
\hline
Dataset  & GCN & GAT & {Sc-only} & Sc-GCN & GSAN \\\hline
COLLAB & 0.592 &0.523  & 0.640 & 0.690 & 0.704\\ \hline
IMBD-BINARY & 0.710 &0.632 & 0.710 & 0.740  & 0.760  \\ \hline
\end{tabular}
\end{table}
\begin{table}[!b]
\renewcommand{\arraystretch}{1.3}
\caption{Comparison of graph regression test error (lower is better).}\label{tab:reg}
\centering
\begin{tabular}{|c|c|c|c|c|c|}
\hline
Dataset   & GCN & GAT & Sc-only & Sc-GCN & GSAN\\\hline
ZINC  & 0.469 &0.463 & 0.51 & 0.452 & 0.430\\ \hline
Lipophilcity  & 1.05 &0.950 & 1.19 & 1.03 & 1.02\\ \hline
\end{tabular}
\end{table}


\noindent\textbf{Graph Regression.}
We use ZINC and Lipophilicity as benchmark datasets \cite{irwin2005zinc,wu2018moleculenet}. For the ZINC dataset, the target is the penalized water-octanol partition coefficient of molecules. For the Lipophilcity dataset, the task is predicting the experimental octanol/water distribution coefficient of different molecules.
ZINC contains 1,000 graphs and Lipophilcity contains 4,200 graphs. Graphs in ZINC have 75 input features and graphs in Lipophilcity have 9. We use a 80-10-10  split  between training,  validation  and  testing.

We again use a GSAN-based approach, this time consisting of two geometric scattering attention layers followed by one graph residual convolution layer. For ZINC, we use a 32-head GSAN with 32 hidden units as the first scattering layer and an one-head GSAN with 128 hidden units as the second layer. For Lipophilcity, we use a 4-head GSAN with 32 hidden units as the first scattering layer and an one-head GSAN  with 96 hidden units as the second layer. The aggregation module uses an MLP, applied to the output node representations, producing a scalar output, followed by an average pooling on each molecule. Again, we provide a scattering-only baseline (Sc-only), where we used a modified version of GSAN, but with all low-pass channels removed, so that our network utilizes only scattering channels. Our results are presented in Tab.~\ref{tab:reg}, demonstrating the utility of our GSAN approach for this task as well. 

\section{Conclusion}\label{sec:conclusion}
Here, we studied GNN approaches for semi-supervised node classification and investigate some of the major limitations of today's architectures. Many popular models are known to essentially rely on the low-pass filtering of graph signals. Based on this observation, we propose a novel hybrid GNN framework that can leverage higher frequency information not captured by traditional GCN models. Our construction is based on geometric scattering, a concept that was previously used mainly for graph classification, which we adapt to the node level.
Our theoretical study suggests that scattering filters nicely balance the trade-off between oversmoothing and underreaching.
Therefore, we are able to theoretically establish that the resulting scattering filters are more sensitive to the graph topology than a large class or GNN architectures including GCN when used in conjunction with node features that encode graph structure information. Empirically, we first evaluated our Sc-GCN model and demonstrated its efficacy in alleviating so-called oversmoothing. We then evaluated our GSAN model, which further improves performance, particularly in more complex (low-homophily) settings, via  an attention framework. We also provide evidence that the proposed hybrid scattering networks perform well in graph-level tasks, both in classification and regression. In future work, one might further explore the potential of features that carry graph structure information empirically, and analyze more datasets in the context of graph-level tasks.

\ifCLASSOPTIONcompsoc
  \section*{Acknowledgments}
\else
  \section*{Acknowledgment}
\fi
%
The authors would like to thank Dongmian Zou for fruitful discussions. This work was partially funded by \mbox{Fin-ML} CREATE graduate studies scholarship for PhD~[\emph{F.W.}]; IVADO (Institut de valorisation des données) grant PRF-2019-3583139727, FRQNT (Fonds de recherche du Québec - Nature et technologies) grant 299376, Canada CIFAR AI Chair~[\emph{G.W.}]; NSF (National Science Foundation) grant DMS-1845856~[\emph{M.H.}]; and NIH (National Institutes of Health) grant NIGMS-R01GM135929~[\emph{M.H.,G.W.}]. The content provided here is solely the responsibility of the authors and does not necessarily represent the official views of the funding agencies.

\ifCLASSOPTIONcaptionsoff
  \newpage
\fi



%



\bibliographystyle{IEEEtran}
\bibliography{references}

%


\begin{IEEEbiography}[{\includegraphics[trim=0in 1.5in 0in 0.5in,width=1in,clip,keepaspectratio]{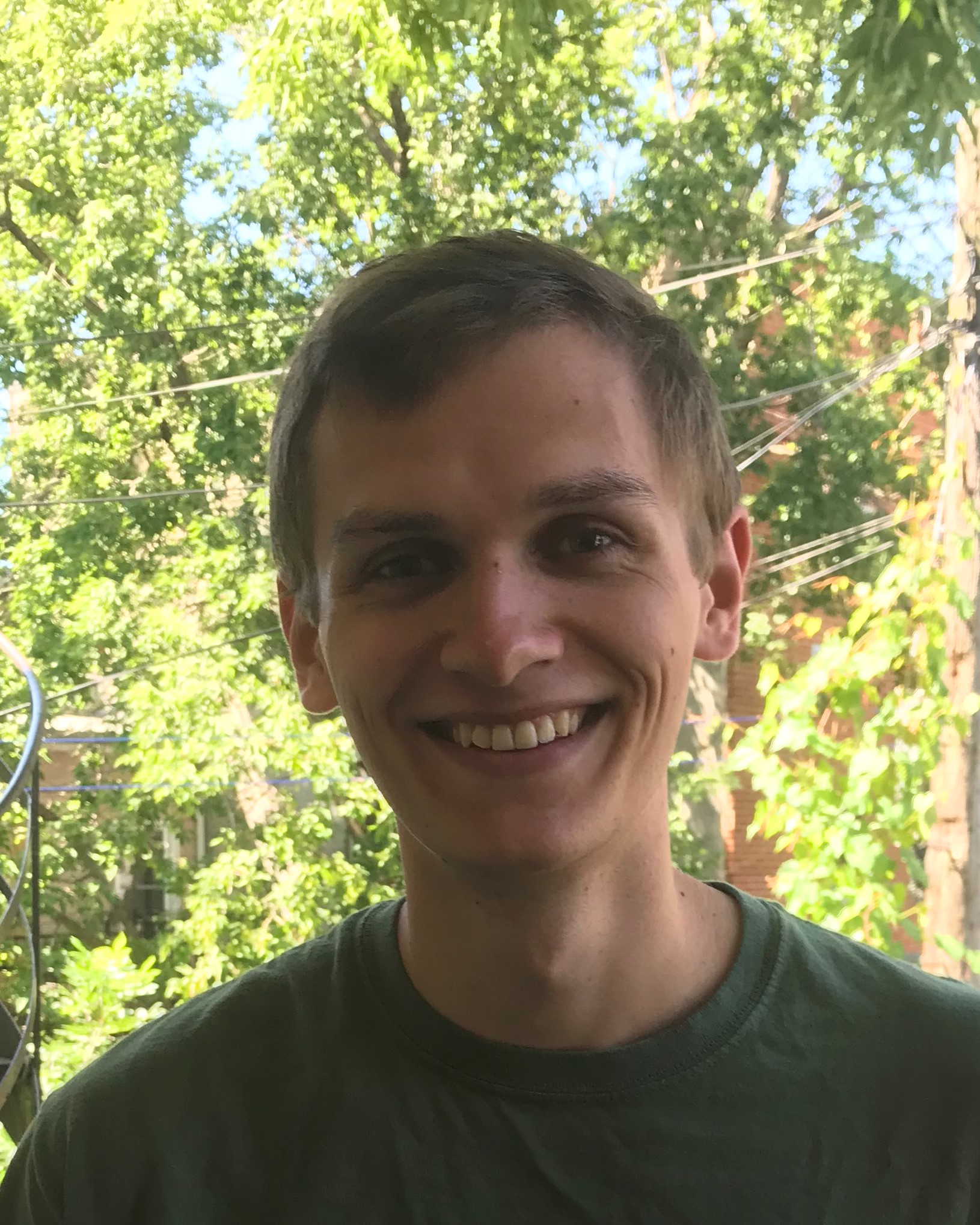}}]{Frederik Wenkel}
received the B.Sc.\ degree in Mathematics and the M.Sc.\ degree in Mathematics at Technical University of Munich, in 2019. He is currently a Ph.D.\ candidate in Applied Mathematics at Universit\'{e} de Montr\'{e}al (UdeM) and Mila (the Quebec AI institute), working on geometric deep learning. In particular, he is interested in graph neural networks and their applications in domains such as social networks, bio-chemistry and finance.
\end{IEEEbiography}

\begin{IEEEbiography}[{\includegraphics[trim=0in 0in 0in 0in,width=1in,clip,keepaspectratio]{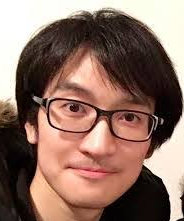}}]{Yimeng Min} is a Ph.D. student in Computer Science at Cornell University. He received his B.Sc.\ degree in Physics from Nanjing University and M.A.Sc.\ degree in Electrical and Computer Engineering from U. of Toronto. His research area is Artificial Intelligence with a focus on graph neural networks and  constraint reasoning.
\end{IEEEbiography}

\begin{IEEEbiography}[{\includegraphics[width=1in,clip,keepaspectratio]{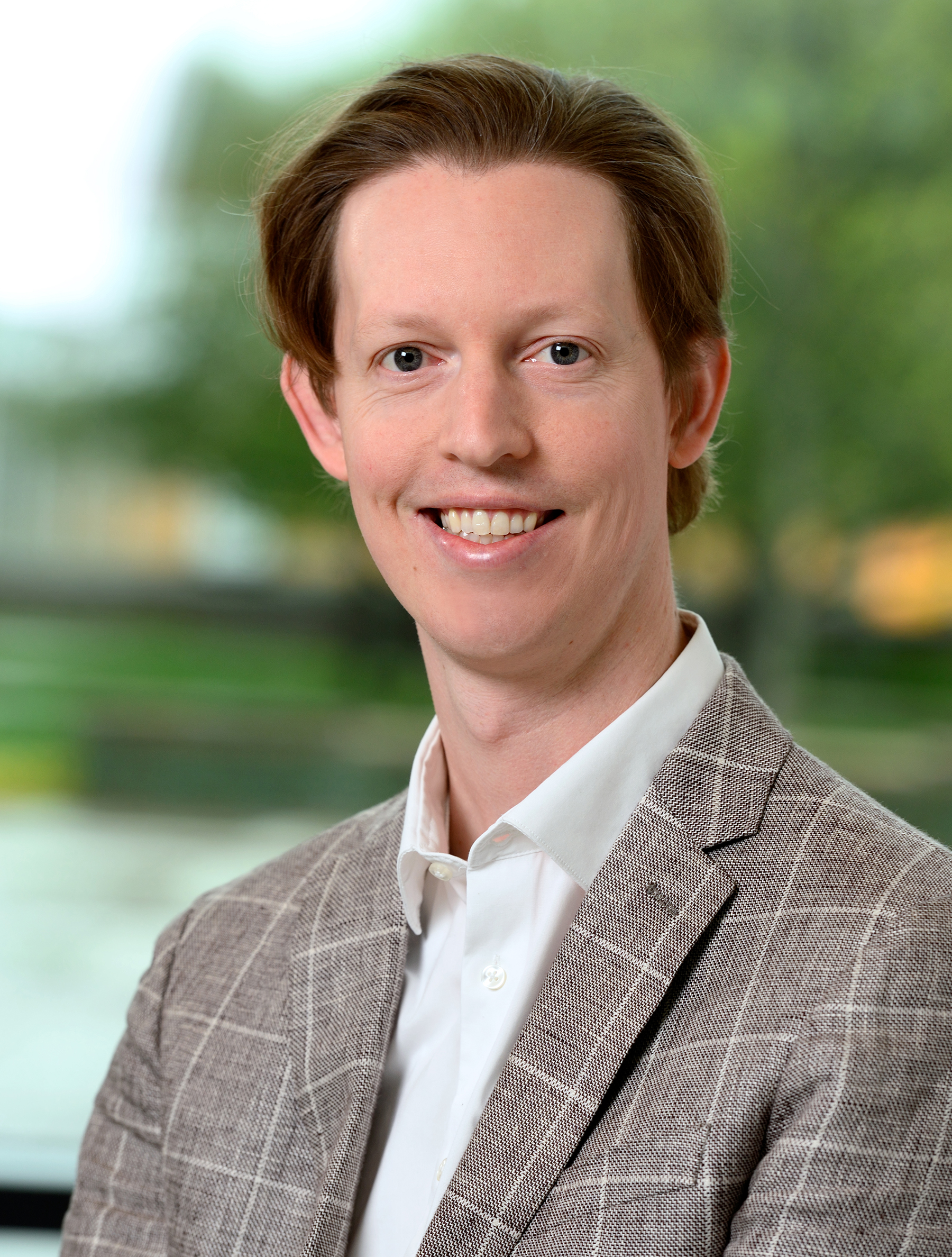}}]{Matthew Hirn} is an Associate Professor in the Dept. of CMSE and the Dept. of Mathematics at Michigan State University. He is the leader of the ComplEx Data Analysis Research (CEDAR) team, which develops new tools in computational harmonic analysis, machine learning, and data science for the analysis of complex, high dimensional data. Hirn received his B.A. in Mathematics from Cornell University and his Ph.D. in Mathematics from the University of Maryland, College Park. Before arriving at MSU, he held postdoctoral appointments in the Applied Math Program at Yale University and in the Department of Computer Science at \'{E}cole Normale Sup\'{e}rieure, Paris. He is the recipient of the Alfred P. Sloan Fellowship (2016), the DARPA Young Faculty Award (2016), the DARPA Director’s Fellowship (2018), and the NSF CAREER award (2019), and was designated a Kavli Fellow by the National Academy of Sciences (2017).
\end{IEEEbiography}

\begin{IEEEbiography}[{\includegraphics[width=1in,clip,keepaspectratio]{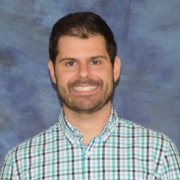}}]{Michael Perlmutter} is a Hedrick Assistant Adjunct Professor in the Dept. of Mathematics at the UCLA. He has also held postdoctoral positions in the Dept. of CMSE at Michigan State University and in the Dept. of Statistics and Operations Research at the University of North Carolina at Chapel Hill. He earned his Ph.D. in Mathematics from Purdue University in 2016. His research uses the methods of applied probability and harmonic analysis to develop and analyze methods for data sets with geometric structure.
\end{IEEEbiography}

\begin{IEEEbiography}[{\includegraphics[trim=0in 1.5in 0in 0.5in,width=1in,clip,keepaspectratio]{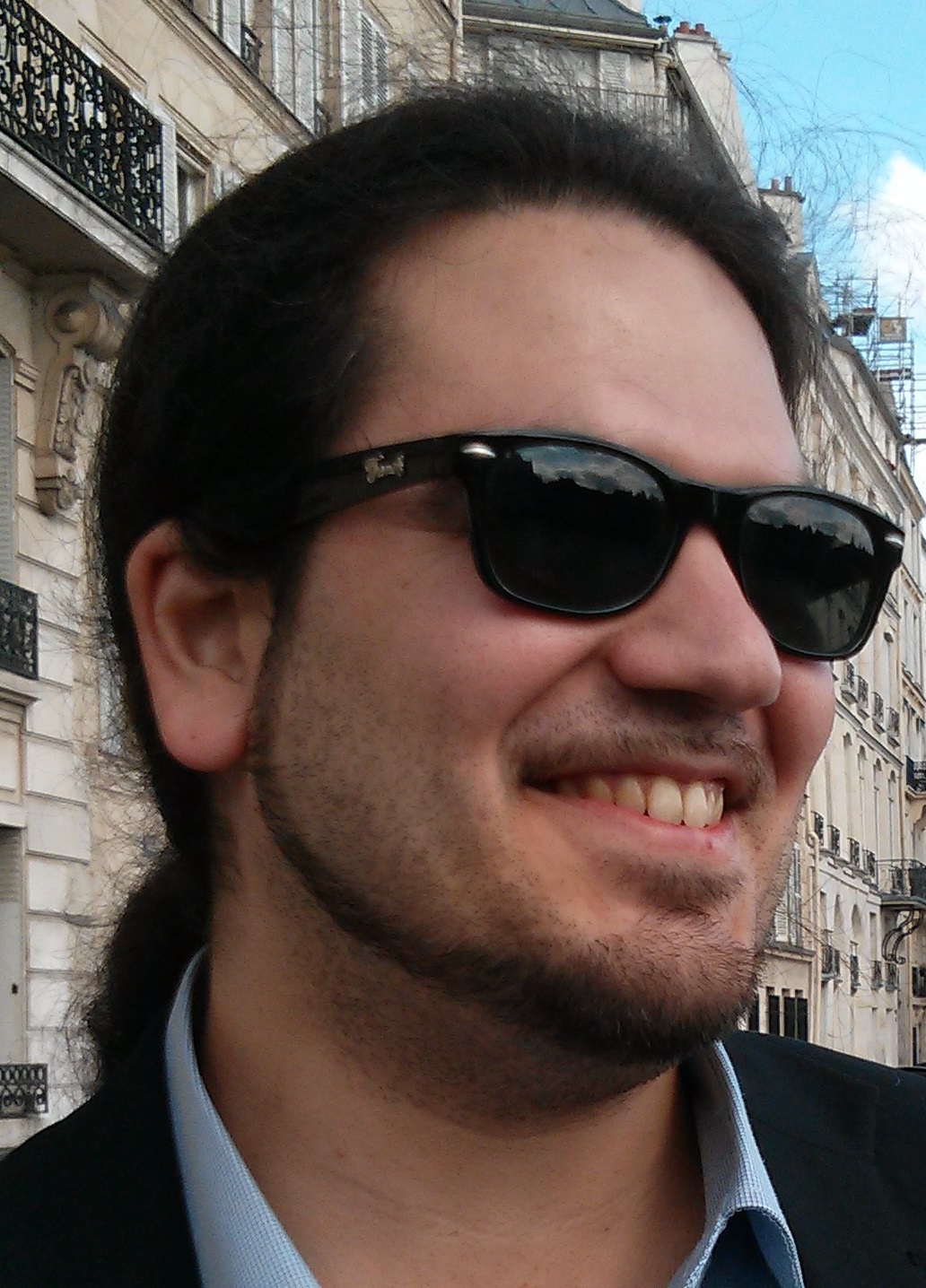}}]{Guy Wolf}
is an associate professor in the Depat. of Mathematics and Statistics (DMS) at the Universit\'{e} de Montr\'{e}al (UdeM), a core academic member of Mila (the Quebec AI institute), and holds a Canada CIFAR AI Chair. He is also affiliated with the CRM center of mathematical sciences and the IVADO institute of data valorization. He holds an M.Sc.\ and a Ph.D.\ in computer science from Tel Aviv University. Previously, he was a postdoctoral researcher (2013-2015) in the Dept. of Computer Science at \'{E}cole Normale Sup\'{e}rieure in Paris (France), and a Gibbs Assistant Professor (2015-2018) in the Applied Math. Program at Yale University. His research focuses on manifold learning and geometric deep learning for exploratory data analysis, including methods for dimensionality reduction, visualization, denoising, data augmentation, and coarse graining. Also, he is particularly interested in biomedical data exploration applications of such methods, e.g., in single cell genomics/proteomics and neuroscience.
\end{IEEEbiography}

\vfill\break

\clearpage
\newpage
\setcounter{page}{1}
\appendices

\section{}

\subsection{Supplemental Definitions for Sec.~\ref{sec:theory}}\label{sec:apx-def}

Here, we provide additional definitions and notations.

\begin{defn}[Node Discriminability]\label{def:discriminability}
    Let $\bX$ be a node feature matrix, and let $\bY=F(\bX)$ for some graph filter $F: \mathbb{R}^{n\times d_X}\rightarrow\mathbb{R}^{n\times d_Y}$. We say that $F$ \textit{discriminates} two  nodes $v,w\in V$ based on $\bX$ if  $\bY[v]\neq \bY[w]$.
\end{defn}

\begin{defn}[K-step Neighborhood]\label{def:K-step}
    For $K\in \mathbb{N}_0$,\footnote{We use the conventions $\N\coloneqq \{1,2,3,\dots\}$, $\N_0\coloneqq \N\cup\{0\}$ and $\N_{\geq k}\coloneqq \{n\in\N : n\geq k\}$, for $k\in\N$.} we define the \textit{K-step neighborhood} of $v$ as
    $
        \nbh_{v}^{K}\coloneqq \{u\in V : 1\leq d(u,v)\leq K \}
    $
    (in particular $\nbh_v^0=\emptyset$ and $\nbh_v^1=\nbh_v$). We further write $\nbh_{\underline{v}}^K\coloneqq \nbh_v^K\cup\{v\}$ (in particular $\nbh_{\underline{v}}^0=\{v\}$).
\end{defn}

\begin{defn}[Induced Subgraph]\label{def:induced}
    Given a set $V_S\subseteq V$, we refer to a subgraph $S=(V_S, E_S)$ of $G$ as an \textit{induced subgraph}, if $E_S\subset E$ is the set of all edges that connect nodes from $V_S$, i.e., $E_S=E(V_S)\coloneqq\{e\in E : e=\{v,w\}\in V_S\times V_S\}$. In this case, we will  write $S\coloneqq G(V_S)$.
\end{defn}

\begin{defn}[Graph Isomorphism]\label{def:isomorph}
    A \textit{graph isomorphism} between graphs $G=(V,E)$ and $G'=(V',E')$ is a bijection $\phi : V\rightarrow V'$ such that $\{v,w\}\in E$ if and only if $\{\phi(v),\phi(w)\}\in E'$. We say that $G$ and $G'$ are $\phi$-isomorphic and write $G\cong^\phi G'$. We also define $\phi(A)\coloneqq\{\phi(w) : w\in A\}$ for $A\subset V$.
\end{defn}

\begin{defn}[Boundary and Interior of Induced Subgraph]\label{def:boundary-interior}
    Let $S=(V_S,E_S)$ be a an induced subgraph of $G=(V,E)$. We define the \textit{boundary} of $S$ as
    $
        \partial S\coloneqq \{w\in V_S : d(w, V\setminus V_S)=1\}.
    $
    Further, we define the \textit{interior} of $S$ as $\interior(S)\coloneqq V_S\setminus \partial S\subset V_S$. When convenient, we will also use the notation $\partial V_S\coloneqq\partial S$ and $\interior(V_S)\coloneqq\interior(S)$.
\end{defn}

\begin{defn}[Spatial Support of Graph Filter]\label{def:spatial}
    We say that a graph filter $F$ has \textit{spatial support} $L\in\N_0$ if the value of $F(\bX)[v]$ depends only on values of $\bX$ at nodes within $L$ steps of $v$, i.e., if for any $v\in V$ and any node feature matrices $\bX_1$ and $\bX_2$, we have  $F(\bX_1)[v]=F(\bX_2)[v]$ whenever $\bX_1[w]=\bX_2[w]$ for all $w\in \mathcal{N}_{\underline{v}}^L$.
\end{defn}

\subsection{Proof of Theorem~\ref{thm:main-1-new}}\label{sec:apx-main-1}

    We use induction to show that 
    $\bX^\ell_u=\bX^\ell_{\phi(u)}$ for all $u\in\nbh_{\underline{v}}^{L-\ell}$
    and all $0\leq \ell\leq L$.
    
    \vspace{8pt}
    \noindent\underline{$\ell=0$:} \,
    Let $u\in\nbh_{\underline{v}}^{L}$. By definition, we have $d(u,v)\leq L,$ and so the triangle inequality implies that $\nbh_{\underline{u}}^{K}\subseteq \nbh_{\underline{v}}^{K+L}.$ This implies that $G\left(\nbh_{\underline{u}}^{K}\right)\cong^\phi G\left(\nbh_{\underline{\phi(u)}}^{K}\right)$ and therefore the assumption that $\bX^0$ is $K$-intrinsic implies that $\bX^0_u=\bX^0_{\phi(u)}$.

    \vspace{8pt}
    
    \noindent\underline{$\ell\rightarrow\ell+1$:}
    \,
    Now assume that $\bX^{\ell}_u=\bX^{\ell}_{\phi(u)}$ for all $u\in\mathcal{N}_{\underline{v}}^{L-\ell}$, and let $w\in\mathcal{N}_{\underline{v}}^{L-(\ell+1)}$. Since $\mathcal{N}_{\underline{v}}^{L-(\ell+1)}\subseteq \mathcal{N}_{\underline{v}}^{L-\ell}$, we have $\bX^\ell_w=\bX^\ell_{\phi(w)}$. Moreover, since the degree is 1-intrinsic, we have $d_w=d_{\phi(w)},$ and therefore $f(\bX^\ell_w,d_w)=f(\bX^\ell_{\phi(w)},d_{\phi(w)})$. Since $\mathcal{N}_{\underline{w}}\subseteq \mathcal{N}_{\underline{v}}^{L},$ we have $G(\mathcal{N}_{\underline{w}}) \cong^\phi G(\mathcal{N}_{\underline{\phi(w)}})$. Therefore, since $\text{agg}$ is a function defined on a multi-set the result will follow from Eq. \ref{eq:ac-gnn} once we show that $g(\bX_u^\ell,d_w,d_u)=g(\bX_{\phi(u)}^\ell,d_{\phi(w)},d_{\phi(u)}$ for all $u\in\mathcal{N}_w$. Towards this end, we note that $\bX_u^\ell=\bX_{\phi(u)}^\ell$ for all $u\in\mathcal{N}_w$ by the inductive assumption and the fact that $\mathcal{N}_{\underline{w}}\subseteq \mathcal{N}_{\underline{v}}^{L-\ell}$. Similarly, since the degree is 1-intrinsic we have $d_w=d_{\phi(w)}$ and $d_{u}=d_{\phi(u)}$, and thus the proof is complete. 
\qed

\subsection{Proof of Lemma~\ref{thm:lem-onion-new}}\label{sec:apx-lem-onion}

We will proceed by \underline{induction over $j$}.
    
\vspace{8pt}

\noindent\underline{$j=0$:} \,
Using the notation established in Notation~\ref{notation-new}, we have that $\widetilde{U}_0=V_{\text{diff}}^{\tilde{d}}$, and as noted in (i), this is exactly the set of nodes in $\mathcal{N}_{\tilde{v}}^{\tilde{d}}$ where a structural difference manifests rel. to $\phi$ and $\btX$. Also note that $\delta_u^0 \neq 0$ for $u\in \tU_0$ according to (iv) of the notation. Thus the lemma holds or $j=0$.

\vspace{8pt}

\noindent\underline{$j\rightarrow j+1$:} \,
Assume Lemma~\ref{thm:lem-onion-new} to hold for some $0\leq j< \td$.
Thus, by the definition of structural difference, and as $d_w = d_{\phi(w)}$ for $w\in \interior\left(\nbh_{\underline{\tv}}^{\td}\right)$, our claim is equivalent to showing
\begin{equation}\label{eq:1-step-diff-new}
    \bY^{j+1}[u]\neq\bY^{j+1}[\phi(u)],
\end{equation}
and $\delta_u^{j+1} \neq 0$ for at least one $u\in \tU_{j+1}$ and
\begin{equation}\label{eq:1-step-eq-new}
    \bY^{j+1}[w]=\bY^{j+1}[\phi(w)],
\end{equation}
for all $w\in\nbh_{\underline{\tv}}^{\td-(j+1)}\setminus \tU_{j+1}$.

Since $\bP^{j+1}\btX=\bP\bP^{j}\btX=\bP\bY^j,$ we may use Eq.~\ref{eq:P-node} to see that 
Eq.~\ref{eq:1-step-diff-new} is equivalent to \begin{align}\label{eq:P-step-2-new}
    & \frac{1}{2} \bY^j[u] + \frac{1}{2} \sum_{w\in\nbh_{u}} d_w^{-1} \bY^j[w] \notag \\
    \neq & \frac{1}{2} \bY^j[\phi(u)] + \frac{1}{2} \sum_{w\in\nbh_{\phi(u)}} d_w^{-1} \bY^j[w]
\end{align}
for at least one $u\in \widetilde{U}_{j+1}\subseteq \mathcal{N}_{\tilde{\underline{v}}}^{d-j}$.  One may check that $\widetilde{U}_{j+1}\cap\widetilde{U}_j=\emptyset$. Therefore, inductively applying Eq.~\ref{eq:1-step-eq-new} (with $j$ in place of $j+1$) yields that $\bY^j[u]=\bY^j[\phi(u)]$. 
    Additionally, the inductive assumption implies that 
$\bY^j[w]=\bY^j[\phi(w)]$
for all $w\in\nbh_{u}\setminus \tU_j$. Therefore, Eq.~\ref{eq:P-step-2-new} is equivalent to
\begin{align}\label{eq:P-step-4-new}
   \sum_{w\in\nbh_{u}\cap \tU_j} d_w^{-1} \bY^j[w] &\neq \sum_{w\in\nbh_{\phi(u)}\cap \phi(\tU_j)} d_w^{-1} \bY^j[w].
\end{align}
This indeed holds for at least one $u\in \tU_{j+1}$ as
Eq.~\ref{eq:P-step-4-new} is equivalent to
\begin{equation}\label{eq:lem-proof-nocc}
    0 \neq \sum_{w\in\nbh_{u}\cap \tU_j} d_w^{-1} \bY^j[w] - d_{\phi(w)}^{-1} \bY^j[\phi(w)] = \delta_u^{j+1},
\end{equation}
and because $\interior\left(\nbh_{\underline{\tv}}^{\td}\right)$ exhibits no-cc with $\Delta_{j+1}\neq \emptyset$ according to the inductive hypothesis.
This completes the proof of Eq. \ref{eq:1-step-diff-new}.

Similarly Eq.~\ref{eq:1-step-eq-new} is equivalent to showing
\begin{align}
    & \frac{1}{2} \bY^j[w] + \frac{1}{2} \sum_{\tw\in\nbh_{w}} d_{\tw}^{-1} \bY^j[\tw] \notag \\
    = & \frac{1}{2} \bY^j[\phi(w)] + \frac{1}{2} \sum_{\tw\in\nbh_{\phi(u_{j+1})}} d_{\tw}^{-1} \bY^j[\tw], \label{eqn: simplied w-new}
\end{align}
for all $w\in\nbh_{\underline{\tv}}^{\td-(j+1)}\setminus \tU_{j+1}$. We first note that $\nbh_{\underline{w}}\subset\nbh_{\underline{\tv}}^{\td-j}$, and that $w\not\in\tU_{j}$, since otherwise we would have $d(\tilde{v},w)=\tilde{d}-j>\tilde{d}-(j+1).$ Therefore, the inductive hypothesis implies $\bY^j[w] = \bY^j[\phi(w)]$. Moreover, no $\tw\in\mathcal{N}_w$ can be an element of $\tU_j$ since otherwise, $w$ would be an element of $\tU_{j+1}$, which is a contradiction. Therefore, the inductive assumption implies $\bY^j[\tw]= \bY^j[\phi(\tw)]$ for all $\tw\in\nbh_w$.
We have already noted that $d_w=d_{\phi(w)}$ for all $w\in\interior\left(\nbh_{\tv}^{\tD}\right)$. Thus, Eq. \ref{eqn: simplied w-new} holds and the proof is complete.
\qed

\subsection{Proof of Theorem~\ref{thm:main-2-new}}\label{sec:apx-main-2}
We need to show that we can choose the parameters in a way that guarantees $F_{p-sct}(\bX)[v]\neq F_{p-sct}(\bX)[\phi(v)]$. For simplicity, we set $\bTheta=\Id_n$. In this case, since $\sigma$ strictly monotonic, and therefore injective, it suffices to show that we can construct $p$ such that 
\begin{equation}\label{eqn: Us are different-new}
    \bU_p(\bX)[v]\neq \bU_p(\bX)[\phi(v)].
\end{equation} 

Using binary expansion, we may choose $k_1,\ldots,k_m\in \mathbb{N}_0$, $k_i<k_{i+1}$, such that $d=2^{k_1}+\ldots 2^{k_m}$.
We will show that Eq. \ref{eqn: Us are different-new} holds for $p\coloneqq(k_1,\ldots,k_m).$ For $1\leq i \leq m,$ let $p_{:i}\coloneqq (k_1,\ldots,k_i)$
and let $t_i=\sum_{j=1}^i 2^{k_j}$, where 
$p_{:0}$  denotes the empty path of length 0 and $t_0=0.$ 

Recall the generalized path $P=(U_0,\ldots,U_d)$ defined in Notation \ref{notation-new} (iii).
We will use induction to show that, for $0\leq i \leq m$, there exists at least one node in $U_{t_i}$ where a structural difference is manifested, while no structural differences manifests in  $\mathcal{N}_{\underline{v}}^{d-t_i} \setminus U_{t_i}$ rel. to $\phi$ and $\bZ^i\coloneqq \bU_{p_{:i}}\bX.$ 
Since $t_m=d$ and $p_{:m}=p$, this claim will imply Eq. \ref{eqn: Us are different-new} and thus prove the theorem. 

\vspace{8pt}

\noindent\underline{$i=0$:} \,
Analogous to the proof of Lemma \ref{thm:lem-onion-new}, using the notation established in Notation~\ref{notation-new} (ii), we have that $U_0=V_{\text{diff}}^{d}$, which are exactly the nodes in $\mathcal{N}_{v}^{d}$ where a structural difference manifests. Thus the claim holds for $i=0$.

\vspace{8pt}

\noindent\underline{$i\rightarrow i+1$:} \,
We now assume the result holds for $i$. Since the degree is one-intrinsic and $t_{i+1}\geq 1$, we have that $d_{\tilde{v}}=d_{\phi(\tilde{v})}$  for all $\tilde{v}\in\mathcal{N}_{\underline{v}}^{d-t_{i+1}}$. Therefore, the claim is equivalent to showing
\begin{equation}\label{eq:wav-step-diff-new}
    \bZ^{i+1}[u]\neq\bZ^{i+1}[\phi(u)],
\end{equation}
for at least one $u\in U_{t_{i+1}}$ and
\begin{equation}\label{eq:wav-step-eq-new}
    \bZ^{i+1}[w]=\bZ^{i+1}[\phi(w)],
\end{equation}
for all $w\in\nbh_{\underline{v}}^{d-t_{i+1}}\setminus U_{t_{i+1}}$.

By the inductive assumption, Eq.~\ref{eq:wav-step-diff-new} and Eq.~\ref{eq:wav-step-eq-new} hold with $i$ in place of $i+1$ and since $\sigma$ is injective, they continue to hold when $\bZ^{i}$ is replaced with $\sigma(\bZ^{i})$. Therefore, we may apply Lemma~\ref{thm:lem-onion-new} with $\sigma(\bZ^i)$ in place of $\btX$, and $j=2^{k_{i+1}}$.
By the inductive assumption, we have that $\tilde{d}$ as defined in Lemma \ref{thm:lem-onion-new} is given by $\tilde{d}=d-t_i$. Therefore, $\td-j=d-t_i-2^{k_{i+1}}=d-t_{i+1}$ and $\tilde{U}_j=U_{t_{i+1}}$. Therefore, there exists at least one $u\in U_{t_{i+1}}$ where a structural difference is manifested, while for all $w\in\nbh_{\underline{v}}^{d-t_{i+1}}\setminus U_{t_{i+1}}$, no structural difference is manifested rel. to $\phi$ and $\bP^{2^{k_{i+1}}}\sigma\left(\bZ^i\right)$, i.e.,
\begin{equation}\label{eqn: udiffwsame1-1-new}
    \bP^{2^{k_{i+1}}}\sigma\left(\bZ^i\right)[u]\neq\bP^{2^{k_{i+1}}}\sigma\left(\bZ^i\right)[\phi(u)]
\end{equation}
and
\begin{equation}\label{eqn: udiffwsame1-2-new}
    \bP^{2^{k_{i+1}}}\sigma\left(\bZ^i\right)[w]=\bP^{2^{k_{i+1}}}\sigma\left(\bZ^i\right)[\phi(w)].
\end{equation}
Next, we again apply  Lemma~\ref{thm:lem-onion-new} with $\sigma(\bZ^i)$ in place of $\btX$, but this time with  $j=2^{k_{i+1}-1}$. We observe now that $\td-j=d-t_i-2^{k_{i+1}-1}>d-t_{i+1}$ and thus $u,w\in \nbh_{\underline{v}}^{\td-t_{i+1}} \subset \nbh_{\underline{v}}^{\td-j}\setminus \tU_j$. Therefore, no structural difference is manifested at either $u$ or $w$ rel. to $\phi$ and $\bP^{2^{k_{i+1}-1}}$, i.e., 
\begin{equation}\label{eqn: usamewsame2-1-new}
    \bP^{2^{k_{i+1}-1}}\sigma\left(\bZ^i\right)[u]=\bP^{2^{k_{i+1}-1}}\sigma\left(\bZ^i\right)[\phi(u)]
\end{equation}
and
\begin{equation}\label{eqn: usamewsame2-2-new}
    \bP^{2^{k_{i+1}-1}}\sigma\left(\bZ^i\right)[w]=\bP^{2^{k_{i+1}-1}}\sigma\left(\bZ^i\right)[\phi(w)].
\end{equation}
Together Eq. \ref{eqn: udiffwsame1-1-new}-\ref{eqn: udiffwsame1-2-new} and Eq. \ref{eqn: usamewsame2-1-new}-\ref{eqn: usamewsame2-2-new} imply
\begin{equation*}
    \bPsi_{k_{i+1}}\sigma\left(\bZ^i\right)[u]\neq\bPsi_{k_{i+1}}\sigma\left(\bZ^i\right)[\phi(u)]
\end{equation*}
and
\begin{equation*}
    \bPsi_{k_{i+1}}\sigma\left(\bZ^i\right)[w]=\bPsi_{k_{i+1}}\sigma\left(\bZ^i\right)[\phi(w)].
\end{equation*}
Eq. \ref{eq:wav-step-diff-new} and \ref{eq:wav-step-eq-new} now follow as
$\bZ^{i+1}=\bPsi_{k_{i+1}} \sigma\left(\bZ^i\right)$.
\qed

\subsection{Modified Variant of Theorem~\ref{thm:main-2-new}}\label{sec:apx-mod-main-2}

\begin{thm}\label{thm:shortest-path}
    Similar to Theorem~\ref{thm:main-2-new}, let $\phi : V \rightarrow V$, $v\in V$ and $K,L\in\N$ such that $G\left(\nbh_{\underline{v}}^{K+L}\right) \cong^\phi G\Big(\nbh_{\underline{\phi(v)}}^{K+L}\Big)$, and consider any K-intrinsic node feature matrix $\bX$.
    Suppose there exist nodes $U\subset\nbh_v^{K+L}$ where a structural difference rel. to $\phi$ and $\bX$ is manifested.
    If the nonlinearity $\sigma$ is strictly monotonic and there exists a unique $u\in U$ with minimal distance from $v$ and a unique shortest path between $u$ and $v$, we can define a scattering configuration $(p,\bTheta)$  such that scattering features $F_{p\psct}(\bX)$ defined as in Eq.~\ref{eq:scat-node} discriminate $v$ and $\phi(v)$. 
\end{thm}

In Theorem~\ref{thm:shortest-path}, we replace the assumption of no coincidental correspondence (Definition~\ref{def:nocc}) in Theorem~\ref{thm:main-2-new} by the assumption of a unique shortest path between $v$ and some node $u$, which is the unique nearest node from $v$, where a structural difference is manifested.

\begin{proof}[Proof of Theorem~\ref{thm:shortest-path}]
    The only difference between Theorem~\ref{thm:shortest-path} and Theorem~\ref{thm:main-2-new} is that in Theorem~\ref{thm:main-2-new} we assume that there is no coincidental correspondence whereas here we assume that there is a unique $u\in U$ with minimal distance from $v$ and that there is a unique shortest path between $v$ and $u$. Inspecting the proof of Theorem~\ref{thm:main-2-new}, we see that the only spot where we used the assumption of no coincidental correspondence was when we invoked Lemma \ref{thm:lem-onion-new}. Therefore, it suffices to show that the conclusion of Lemma \ref{thm:lem-onion-new} holds under our revised assumptions. Moreover, inspecting the proof of Lemma \ref{thm:lem-onion-new}, we see that the only spot we use the assumption of no coincidental correspondence was to establish Eq.~\ref{eq:lem-proof-nocc}. Thus, it suffices to show that Eq.~\ref{eq:lem-proof-nocc} still holds.
    
    Let $P\coloneqq (U_0,U_1,\dots,U_{d})$ be the generalized path from $u$ to $v$ as defined in Notation \ref{notation-new}, and note that our assumption of a unique shortest path implies that each $U_j$ consists of a single node $u_j$, i.e., $U_j \coloneqq \{u_j\}$. Since the assumption of no coincidental correspondence is not used in the base case, we can use the inductive hypothesis that $u_j$ is the unique element of $\mathcal{N}_{\underline{v}}^{d-j}$ where a structural difference manifests with respect to $\bY^j$. We are then left to show that this implies that Eq.~\ref{eq:lem-proof-nocc} holds.
    But Eq.~\ref{eq:lem-proof-nocc} simplifies to $0 \neq d_{u_j}^{-1} \bY^j[u_j] - d_{\phi(u_j)}^{-1} \bY^j[\phi(u_j)]$, which is true according to the inductive hypothesis.
\end{proof}

\section{}\label{sec:apx-empirical}
\subsection{Technical Details}

\begin{table*}[!t]
    \centering
    \caption{Chosen hyperparameters for our Sc-GCN model on the examined datasets.}\label{tab:hyperparameters}
    \begin{tabular}{|c||c|c|cc|ccccc|}
        \multicolumn{3}{c}{~} & \multicolumn{2}{c}{$\overbrace{\hspace{60pt}}^\text{Scat.\ config.:}$} & \multicolumn{5}{c}{$\overbrace{\hspace{130pt}}^\text{Channel widths:}$}\\
        \cline{2-10}
        \multicolumn{1}{c}{~} & \multicolumn{1}{|c}{$\alpha$} & \multicolumn{1}{|c}{$q$} & \multicolumn{1}{|c}{$\boldsymbol{U}_{p_1}$} & \multicolumn{1}{c}{$\boldsymbol{U}_{p_2}$} & \multicolumn{1}{|c}{$\,\boldsymbol{A}^1\,$} & \multicolumn{1}{c}{$\,\boldsymbol{A}^2\,$} & \multicolumn{1}{c}{$\,\boldsymbol{A}^3\,$} & \multicolumn{1}{c}{$\boldsymbol{U}_{p_1}$} & \multicolumn{1}{c|}{$\boldsymbol{U}_{p_2}$} \\
        \hline
        \emph{Citeseer} & 0.50 & 4 & $\boldsymbol{\Psi}_2$& $\boldsymbol{\Psi}_2|\boldsymbol{\Psi}_3|$  & 10 & 10 & 10 & 9 & 30 \\
        \emph{Cora} & 0.35 & 4 & $\boldsymbol{\Psi}_1 $& $\boldsymbol{\Psi}_3$ & 10 & 10 & 10 & 11 & 6\\
        \emph{Pubmed} & 1.00 & 4 & $\boldsymbol{\Psi}_1 $& $\boldsymbol{\Psi}_2$ & 10 & 10 & 10 & 13 & 14 \\
        \emph{DBLP (1st layer)} & 1.00 & 4 & $\boldsymbol{\Psi}_1 $& $\boldsymbol{\Psi}_2$ & 10 & 10 & 10 & 30 & 30 \\
        \emph{DBLP (2nd layer)} & 0.10 & 1 & $\boldsymbol{\Psi}_1 $  & $\boldsymbol{\Psi}_2 $  & 40 & 20 & 20 & 20 & 20 \\
        \hline
    \end{tabular}
\end{table*}

\begin{table}[!b]
    \centering
    \caption{Chosen hyperparameters for our GSAN model on the examined datasets.}\label{tab:GSANhyperparameters}
    \begin{tabular}{|c||c|c|c|}
        \cline{2-4}
        \multicolumn{1}{c}{~} & \multicolumn{1}{|c}{$\alpha$} & \multicolumn{1}{|c|}{Number of heads} & \multicolumn{1}{c|}{Channel  width } \\
        \hline
        \emph{Texas} & 1.5 & 4 & 128 \\
        \emph{Chameleon} & 0.2 & 64 & 32 \\
        \emph{CoraFull} & 0.5 & 20 & 128 \\
        \emph{Wiki-CS} & 0.3 & 20 & 20 \\
        \emph{Citeseer} & 0.1 & 8 & 64 \\
        \emph{Pubmed} & 0.1 & 50 & 64 \\
        \emph{Cora} & 0.1 & 50 & 64 \\
        \emph{DBLP} & 0.2 & 16 & 128 \\
        \hline
    \end{tabular}
\end{table}

As with most neural networks, when implementing Sc-GCN and GSAN, one must make several architecture choices and tune several hyperparameters. In our experiments with Sc-GCN, we used either one or two hybrid layers, each comprised of three GCN channels and two scattering channels, followed by a single residual convolution layer. This fairly simple setup  simplified the network tuning process and was sufficient to obtain strong results, which outperformed numerous competing methods. However, it is likely that performance can be further improved by using wider or deeper networks. For \emph{Cora}, \emph{Citeseer} and \emph{Pubmed} we used only one hybrid layer as preliminary experiments indicated that the addition of a second one was not cost-effective (in light of  additional complexity created by the need to tune more hyperparameters). For \emph{DBLP}, two layers were used due to a significant increase in performance. We note that even with a single hybrid layer our model achieves $73.1\%$ test accuracy (compared to  $81.5\%$ with two layers) and still significantly outperforms GAT ($66.1\%$) and the other methods (below $60\%$). \\

To illustrate the utility of our model on large and challenging data sets, we now compare against GCN and GAT on two OGB benchmark node classification tasks. Our results are shown in Tab~\ref{tab:ogb_accu}. GSAN achieves 73.9\% using 3,832 parameters on ogbn-proteins and 72.3\% on ogbn-arxiv using 70,760 parameters, while GCN achieves  72.5\% ogbn-proteins and using 110,120 parameters and 71.7\% on ogbn-arxiv using 96,880 parameters, GAT achieves  73.6\% on ogbn-proteins using 1,238,554 parameters and 71.2\% on on ogbn-arxiv using 950,620 parameters. This suggests that our model can be more parameter-efficient on large datasets, especially on ogbn-proteins dataset, where we only use 3.5\% of the GCN's parameter counts. Furthermore, GSAN achieves 76.2\% on ogbn-proteins using 31,856 when we increase the width of GSAN. Thus, we see that our network is able to get better performance with fewer learned parameters.
\begin{table}[!ht]
\renewcommand{\arraystretch}{1.1}
\caption{Comparison of node classification test accuracy on Open Graph Benchmark datasets proteins and arxiv (higher is better).}\label{tab:ogb_accu}
\centering
\begin{tabular}{|c|c|c|c|c|c|c|}
\hline
Dataset  & \#nodes & \#edges & GCN & GAT & Sc-GCN & GSAN \\\hline
proteins & 1.3e5	 & 4.0e7 & 0.725 & 0.736  & 0.742 & 0.739  \\ \hline
arxiv & 1.7e5	 & 1.1e6	 & 0.717 & 0.712 & 0.719 & 0.723\\ \hline
\end{tabular}
\end{table}\\

\noindent\textbf{Validation \& testing procedure:} All tests were done using train-validation-test splits of the datasets. We used the validation accuracy  to tune the hyperparameters and reported the test accuracy in the comparison tables. To ensure fair comparison, we used the same splits for all methods. On \emph{Citeseer}, \emph{Cora} and \emph{Pubmed}, we used the same settings as in~\cite{kipf2016semi}, and followed the standard practices used in other works which have used these datasets. For \emph{DBLP}, as far as we know, no common standard is established in the literature.  Here, we used a ratio of $5:1:1$ between train, validation, and test for all node-classification datasets. For the datasets for graph classification and regression, we used a ratio of $8:1:1$.\\ 

\noindent\textbf{Hyperparameter tuning:} We tuned our hyperparameters on each set using a grid search and selected the setting which yields the highest accuracy on the validation set.

In Sc-GCN, we used the grid search to tune the parameter $\alpha$ used in the residual convolution, the exponent $q$ used in the nonlinearity, the scattering channel configurations $p_1, p_2$ (i.e., scales used in these two channels), and the widths of channels in the network, i.e. the hidden dimensions. The results of this tuning process are presented in Tab.~\ref{tab:hyperparameters}. For \emph{DBLP}, the hybrid-layer parameters are shared between the two layers in order to simplify the tuning process. This was generally less exhaustive than the process used for the other three datasets, since  our method significantly outperformed all other methods even with limited tuning.

\begin{table}
\begin{minipage}{0.48\textwidth}
\caption{Classification accuracies on ogbn-proteins dataset with different $\alpha$. }\label{tab:alpha_parameter}
\centering
\begin{small}
\begin{sc}
\begin{tabular}{|c||c|c|c|}
\hline
$\alpha$  & 0.01 & 0.1 & 0.5  \\\hline
Accuracy  & $72.32 \pm 0.55$ & $72.03 \pm 0.81$ & $72.45 \pm 1.30$  \\ \hline
$\alpha$  & 1.0 & 3.0 & 5.0  \\\hline
Accuracy  & $73.93 \pm 0.19$  & $72.23 \pm 0.45$ & $71.45 \pm 1.41$ \\ \hline
\end{tabular}
\end{sc}
\end{small}
\end{minipage}
\end{table}

For GSAN, the hyperparameter selection is less intricate, which in addition to improved performance, is one of the main advantages compared to Sc-GCN. In particular, we use the same scattering configurations and channel widths across all channels in a given layer. We do not tune the scattering configurations. We always set $m=1$ and always use the wavelets $\bPsi_1, \bPsi_2,$ and $\bPsi_3$ in our three scattering channels. Therefore, the hyperparameters simplify to $\alpha$, a (universal) channel width and the number of attention heads. These parameters might differ across different layers if the architecture relies on several ones. The results of this tuning process are presented in Tab.~\ref{tab:GSANhyperparameters}. \\

\noindent\textbf{Hardware \& software environment:}
All experiments were done on the same HPC cluster with intel i7-6850K CPU and NVIDIA TITAN X Pascal GPU. Both Sc-GCN and GSAN were implemented 
using PyTorch \cite{paszke2019pytorch}. Implementations of all other methods were taken directly from the code accompanying their publications.

\subsection{Ablation Study}\label{sec:apx-ablation}
The two primary novelties in Sc-GCN and GSAN are the  scattering channels (i.e., $\bPsi_1, \bPsi_2,$ and $\bPsi_3$) and the residual convolution (controlled by the hyperparameter $\alpha$). To explore their contribution and the dependence of the networks performance on the hyperparameters, Tab.~\ref{tab:alpha_parameter} show classification results over the ogbn-proteins dataset for $\alpha=0.01,0.1,0.5,1.0,3.0, 5.0$ over multiple scattering channel configurations. For  simplicity, we focus on \emph{ogbn-proteins}, but note that similar results are  observed on the other datasets. 

\newpage

\begin{table}[!ht]
    \caption{Impact of removing each individual channel from the optimal configuration on ogbn-proteins, while classifying using the remaining four channels. Full GSAN accuracy is 73.93\% .}
    \label{table:ablation_channnels}
    \centering
    \begin{tabular}{|c||cccccc|}
        \hline
         Channel &  $\bA $    & $\bA^2$     & $\bA^3 $  & $\bPsi_3$  & $\bPsi_2$     & $\bPsi_1$  \\
        \hline
        Accuracy & 71.22 & 71.23 & 72.16 & 71.47 & 72.98 & 70.75  \\
        \hline
    \end{tabular}
\end{table}

To examine the importance of the parameter $\alpha$ used in the residual graph convolution layer, we observe that setting $\alpha = 0$ effectively eliminates this layer (since then $\bA_{\text{res}}=\Id_n$). On the other hand, increasing $\alpha$ makes the filtering effect stronger, approaching a random-walk filtering as $\alpha\rightarrow \infty$. We evaluate the effect of this parameter on classification accuracy.
Our results, displayed in Tab.~\ref{tab:alpha_parameter}, indicate that increasing $\alpha$ to non-negligible nonzero values improves classification performance, which we interpret to be due to the removal of high-frequency noise. However, when $\alpha$ further increases (in particular when $\alpha=5$ in this case) the smoothing provided by this layer degrades the performance to a level close to the traditional GCN~\cite{kipf2016semi}. 

Finally, to examine the relative importance of the low-pass and  band-pass information, we provide an ablation study of the impact each channel. We focus on the ogbn-proteins dataset, set $\alpha=1.0$ and use the best channel configuration, which achieved 73.93\% accuracy when all channels were used. Then, we remove each of the band-pass or low-pass channels individually from the network, and reevaluate network performance with the remaining four channels. Our results are displayed in Tab.~\ref{table:ablation_channnels}. They indicate that the information captured by both the band-pass and the low-pass channels is important and that eliminating either decreases performance. In particular, we note that eliminating $\boldsymbol{\Psi}_1$  has a major impact on the accuracy.

\subsection{Implementation}
Python code accompanying this work is available on \href{http://github.com/dms-net/scatteringGCN}{github.com/dms-net/scatteringGCN} and \href{https://github.com/dms-net/Attention-based-Scattering}{github.com/dms-net/Attention-based-Scattering}.

\end{document}